\documentclass[runningheads]{llncs}
\usepackage[T1]{fontenc}
\usepackage{graphicx}
\usepackage{times}  % DO NOT CHANGE THIS
\usepackage{helvet}  % DO NOT CHANGE THIS
\usepackage{courier}  % DO NOT CHANGE THIS
\usepackage[hyphens]{url}  % DO NOT CHANGE THIS

\usepackage{natbib}  % DO NOT CHANGE THIS AND DO NOT ADD ANY OPTIONS TO IT
\usepackage{caption} % DO NOT CHANGE THIS AND DO NOT ADD ANY OPTIONS TO IT
\usepackage{algpseudocode}
\usepackage{algorithm}
\usepackage{amssymb}
\usepackage{amsmath}
\usepackage{bm}
\usepackage{booktabs}
\usepackage{cuted}
\usepackage{listings}
\usepackage{mathtools}
\usepackage{multirow}
\usepackage{newfloat}
\usepackage{subcaption}
\usepackage{threeparttable}
\usepackage{balance}
\usepackage{textcomp}
\usepackage{pgfplots}
\usepackage{tikz}
\pgfplotsset{compat=1.18}
\usepackage{xcolor}
\usepackage{csquotes}
\usepackage{hyperref} % REMOVED BEFORE SUBMISSION
\usepackage[nameinlink, capitalize, noabbrev]{cleveref}

\DeclareCaptionStyle{ruled}{labelfont=normalfont, labelsep=colon, strut=off} % DO NOT CHANGE THIS
\floatstyle{ruled}
\newfloat{listing}{tb}{lst}{}
\floatname{listing}{Listing}

\usetikzlibrary{arrows.meta}
\tikzset{
    node/.style={draw, thick, circle, minimum size=0.75cm, font=\tiny},
    edge/.style={thick, -Stealth}
}

\newcommand{\eqperiod}{\text{.}}
\newcommand{\eqcomma}{\text{,}}
\newcommand{\prob}{\mathbb{P}}
\newcommand{\paset}{\mathcal{P}}
\newcommand{\init}{\mathrm{init}}
\DeclareMathOperator*{\argmin}{argmin}
\DeclareMathOperator*{\argmax}{argmax}
\DeclareMathOperator{\subjto}{s.t.}

\begin{document}
\title{Structure-Aware Robust Counterfactual Explanations via Conditional Gaussian Network Classifiers}
\titlerunning{Structure-Aware Robust Counterfactual Explanations}
\author{Zhan-Yi Liao\inst{1}  % \orcidID{0000-1111-2222-3333} 
 \and
Jaewon Yoo\inst{2} \and
Hao-Tsung Yang\inst{3} \and Po-An Chen\inst{1}}
\authorrunning{Z.-Y. Liao et al.}
\institute{National Yang Ming Chiao Tung University, Hsinchu, Taiwan\\ \email{\{zyliao.mg12, poanchen\}@nycu.edu.tw} \and
National Tsing Hua University, Hsinchu, Taiwan \\
\email{jaewon.yoo@iss.nthu.edu.tw} \and
National Central University, Taoyuan, Taiwan \\
\email{haotsungyang@gmail.com}}
\maketitle % typeset the header of the contribution
%
%\textcolor{red}{(Please highlight all revisions and edits in red.)}

% 為何貝氏分類器優於黑箱
% 為何 CGNC 較 interpretable（例子）
% 為何離散 CE 不可行、必須連續化

\begin{abstract}

% Counterfactual explanation (CE), a core technique in explainable AI (XAI), is widely used to interpret model decisions and suggest actionable alternatives. This paper proposes a structure-aware and robust counterfactual search method based on the Conditional Gaussian Network Classifier (CGNC). By modeling conditional dependencies among features via a graphical structure, CGNC naturally embeds variable relationships and generative logic into the search process, eliminating the need for additional feasibility constraints. Unlike most black-box approaches that overlook feature dependencies and data distributions, our method integrates these properties with perturbation robustness into the optimization process, ensuring consistency with the model's structure of feature-dependencies. We employ a cutting-set algorithm to iteratively solve a master-adversary optimization problem, which terminates in a finite number of iterations. Starting from a nonconvex quadratic program introduced by the model, a piecewise relaxation can be applied to handle second-order terms, enabling reformulation into a mixed-integer linear program (MILP). Experiments show that even a global optimization of the original nonconvex quadratic problem yields stable and efficient results. The framework is extensible to more complex constraint settings and future work on other nonconvex quadratic programming designs.

Counterfactual explanation (CE) is a core technique in explainable artificial intelligence (XAI), widely used to interpret model decisions and suggest actionable alternatives. This work presents a structure-aware and robustness-oriented counterfactual search method based on the conditional Gaussian network classifier (CGNC). The CGNC has a generative structure that encodes conditional dependencies and potential causal relations among features through a directed acyclic graph (DAG). This structure naturally embeds feature relationships into the search process, eliminating the need for additional constraints to ensure consistency with the model's structural assumptions. We adopt a convergence-guaranteed cutting-set procedure as an adversarial optimization framework, which iteratively approximates solutions that satisfy global robustness conditions. To address the nonconvex quadratic structure induced by feature dependencies, we apply piecewise McCormick relaxation to reformulate the problem as a mixed-integer linear program (MILP), ensuring global optimality. Experimental results show that our method achieves strong robustness, with direct global optimization of the original formulation providing especially stable and efficient results. The proposed framework is extensible to more complex constraint settings, laying the groundwork for future advances in counterfactual reasoning under nonconvex quadratic formulations.


\end{abstract}

\section{Introduction}
\label{sec:introduction}

Explainable artificial intelligence (XAI) seeks to improve the transparency and trustworthiness of machine learning models, allowing developers, users, and regulators to understand the rationale behind predictions \cite{barredoarrieta2020}. It has been widely adopted in high-stakes domains such as judicial decision-making, medical diagnosis, and financial risk management to ensure compliance and build trust \cite{deeks2019, vandervelden2022, cerneviciene2024}. However, traditional methods, including feature importance rankings and visualization, offer only partial insights and lack actionable guidance, limiting their practical utility \cite{ribeiro2016, lundberg2017}.

Recent work has therefore shifted toward actionable explanations, with \emph{counterfactual explanation} (CE) emerging as a key approach for its operational relevance. CE generates minimally altered instances that flip a model's prediction, formulated as an optimization problem \cite{wachter2018}. This allows users to ask questions like \enquote{What needs to change to achieve a desired outcome?} This facilitates model understanding and decision-making, as in credit scoring, where CE can both explain a denial and indicate which features should be modified for approval.

Despite this promise, most CE methods rely on black-box models such as neural networks and decision trees, making it difficult to assess their logic, reliability, and risks \cite{ghorbani2019, saeed2023}. Furthermore, many approaches generate explanations solely from input--output behavior after training, which can yield recommendations that violate the data-generating process and lack causal or semantic consistency. %\textcolor{red}{
In counterfactual search, this limitation is particularly critical, as explanations cannot be structurally verified against the model itself, making their validity difficult to assess beyond empirical behavior.%}

To overcome these limitations, this study proposes a robust counterfactual search method built on the \emph{conditional Gaussian network classifier} (CGNC). CGNCs provide explicit graphical structures and parameterized conditional Gaussian distributions that capture feature dependencies, offering a transparent and verifiable foundation. On top of this structure, we incorporate an adversarial optimization framework to enhance robustness. Together, these elements enable counterfactuals that remain faithful to the CGNC and are robust against perturbations.

\subsection{Related Work}
\label{subsec:relatedwork}

Although CE aims to alter predictions through minimal input changes, these changes are often unstable and sensitive to small variations such as noise, model updates, or feature mislabeling, which limits their reliability in practice \cite{slack2021, pawelczyk2022, dominguezolmedo2022}. Early attempts to improve feasibility introduced objectives that generate multiple candidate CEs simultaneously \cite{russell2019, karimi2020}. However, without stability or structural constraints, many of these explanations become implausible relative to the data, reducing their usefulness for actionable decision support.

To improve reliability, recent studies have examined robustness as a key property, focusing on sensitivities to input perturbations, model changes, and hyperparameter selection \cite{mishra2023}. Our work instead emphasizes \emph{robustness in recourse actions}, ensuring that predictions remain flipped even when the implemented actions deviate slightly from the proposed adjustments, as often occurs when users act under uncertain or imperfect conditions \cite{pawelczyk2022}.

Several studies have sought to strengthen recourse robustness by employing adversarial optimization frameworks that embed worst-case deviations into the feasible set design \cite{virgolin2023, maragno2024}. Nevertheless, many of these methods remain post-hoc, assessing solution quality afterward with metrics such as proximity, connectedness, and stability \cite{laugel2019}, rather than incorporating structural or causal constraints directly into the optimization. As a result, the generated counterfactuals often lack practical feasibility and suggest unrealistic changes.

A recent survey notes that desirable properties of CEs include validity, minimality, plausibility, and \emph{causality}, yet most methods cover only a subset of them \cite{guidotti2022}. To bridge this gap, algorithmic recourse emphasizes embedding quality metrics directly into the optimization phase \cite{karimi2021b}, for instance, by integrating feasible actions and cost functions to produce recourse sets \cite{ustun2019}. Later works use structural causal models (SCMs) to impose constraints \cite{mahajan2020, karimi2021a}, incorporate Mahalanobis distances \cite{kanamori2020}, or enforce feature-changing orderings \cite{kanamori2021}, aiming for optimization models better aligned with distributional or structural assumptions. These enhancements highlight a promising direction for advancing counterfactual design.

\subsection{Research Challenges and Contributions}

As CE techniques expand across application domains, most existing methods still assume feature independence, often leading to implausible recommendations. For example, a counterfactual might suggest increasing education level without adjusting age, disregarding their natural dependency. The key challenge is therefore to design counterfactual search methods that align with the data-generating process and produce structurally coherent recommendations.

As discussed in \cref{subsec:relatedwork}, our notion of causality refers to structural dependencies among features encoded by a \emph{directed acyclic graph} (DAG), rather than intervention-based causal reasoning \cite{guidotti2022}. In contrast, CE methods flip predictions by minimizing input changes under a fixed model, without ensuring validity under realistic interventions or relying on frameworks like SCMs or Pearl's input--output formulation \cite{pearl2009}. Building on this distinction, the present work focuses on model-intrinsic CE methods that account for feature dependencies.

To address these limitations, generative models that encode feature dependencies are needed to keep counterfactual search consistent with the data and model semantics. Bayesian networks (BNs) provide a natural starting point, and the previous study applied Bayesian network classifiers (BNCs) to counterfactual search with discrete features \cite{albini2020}. %\textcolor{red}{
However, these methods do not extend to continuous dependencies, which prevents counterfactual search from being cast as a continuous optimization problem necessary for defining distances and robustness under uncertainty. This study adopts the CGNC, in which the class label serves as a common parent and a DAG explicitly encodes conditional Gaussian dependencies among features. %} 
This structure supports interpretability and generative consistency, making CGNC a suitable basis for structurally aligned counterfactual generation.

Building on this model, we propose a structure-aware counterfactual search method that emphasizes robust recourse. By casting the task as an optimization problem, structural constraints and robustness requirements are directly embedded into the search. This ensures that generated counterfactuals remain consistent with the underlying generative model while remaining reliable under implementation errors \cite{justin2025}.

The key contributions of this work are summarized below:
\begin{itemize}
    \item %\textcolor{red}{
    We present the first structure-aware counterfactual search framework built on a \emph{continuous} BNC, formulating counterfactual generation on the CGNC as a constrained optimization problem that explicitly embeds conditional Gaussian dependencies.%}
    \item We leverage class-conditional \emph{whitening} and Mahalanobis-$\ell_p$ distances to define distribution-consistent distance metrics and adversarial \emph{uncertainty sets}, enabling robust recourse under perturbations.
    \item We employ a \emph{cutting-set procedure} as the adversarial optimization framework and establish convergence guarantees of the iterative process, providing theoretical support for robust counterfactual recourse.
    \item We analyze the nonconvex quadratic structure of the CGNC decision function and derive a tractable mixed-integer linear programming (MILP) formulation via \emph{piecewise McCormick relaxations}, enabling practical computation.
    \item We conduct empirical studies, including comparative evaluations, to examine how modeling feature dependencies affects counterfactual quality and robustness.
\end{itemize}

% \subsection{Organization of the Thesis}

% The outline of the thesis is as follows. \cref{ch:preliminaries} introduces traditional BNs and their extensions, and establishes the theoretical foundation of the CGNC. \cref{ch:methodology} formulates a robust counterfactual search method based on the CGNC, incorporating distance metrics and the adversarial optimization framework. \cref{ch:algorithm} develops the full formulation, derives the nonconvex quadratic structure, and shows how it can be reformulated as a MILP to bypass direct nonconvex solving. \cref{ch:experiments} presents the dataset, structure learning, and evaluates performance with feature dependencies and under different parameter settings. Finally, \cref{ch:conclusion} summarizes the findings, discusses limitations, and outlines directions for future work.

\subsection{Bayesian Network \& Bayesian Network Classifier}

A BN encodes the joint distribution of a set of random variables using a DAG, where each node corresponds to a variable and edges represent conditional dependencies. As a Bayesian framework, it supports flexible modeling of uncertainty and facilitates decision-making, while the graph structure enables the factorization of the joint distribution, leading to a compact representation and tractable inference \cite{koller2009, murphy2012}.

In classification settings, a BNC introduces a discrete categorical class variable $Y \in \mathcal{Y}$ as a common parent of all feature variables $\bm{X} = (X_1, \dots, X_n)$ \cite{friedman1997}. Each feature variable $X_i$ for $i \in [n]$ may additionally have a subset of other feature variables as parents, denoted by $\bm{X}_{\paset_i}$. Here, $[n] \coloneq \{1, \dots, n\}$ denotes the set of integers from $1$ to $n$, and $\paset_i \subset [n]$ is the parent index set for node $X_i$. A variable $X_j$ is considered a parent of $X_i$ if and only if $j \in \paset_i$. Under this structure, the class-conditional distribution for class $c$ factorizes as
\begin{equation}
    \label{bnc:dist}
    \prob(\bm{X} \mid Y = c) = \prod_{i=1}^{n} \prob(X_i \mid \bm{X}_{\paset_i}, Y = c) \eqperiod
\end{equation}

For a given input $\bm{x} = (x_1, \dots, x_n)$ and any class $c$, the posterior probability follows Bayes' theorem:
\begin{equation}
    \label{eq:post_prob}
    \prob(Y=c \mid \bm{x}) = \frac{\prob(Y=c) \cdot \prob(\bm{x} \mid Y=c)}{\prob(\bm{x})} \eqperiod
\end{equation}
In the \emph{maximum a posteriori} (MAP) rule, the denominator $\prob(\bm{x})$ is independent of $c$ and thus can be omitted. Using the factorized form of $\prob(\bm{x} \mid Y=c)$, the MAP classifier becomes
\begin{equation}
    \label{eq:bnc_classif}
    \hat{c} = \argmax_{c \in \mathcal{Y}} \; \prob(Y=c) \cdot \prod_{i=1}^{n} \prob(x_i \mid \bm{x}_{\paset_i}, Y=c) \eqperiod
\end{equation}
This formulation directly follows from the conditional independence structure of the network.

We consider three generative classifiers that differ in their ability to capture feature dependencies, with increasing structural flexibility: na{\"i}ve Bayes (NB), tree-augmented na{\"i}ve Bayes (TAN), and Bayesian network-augmented na{\"i}ve Bayes (BAN), as illustrated in \cref{fig:main-comparison}. In these diagrams, dashed edges represent class-to-feature dependencies, while solid edges denote additional conditional dependencies among features allowed under each structure. Each model progressively relaxes the conditional independence assumptions of the NB framework \cite{friedman1997}.

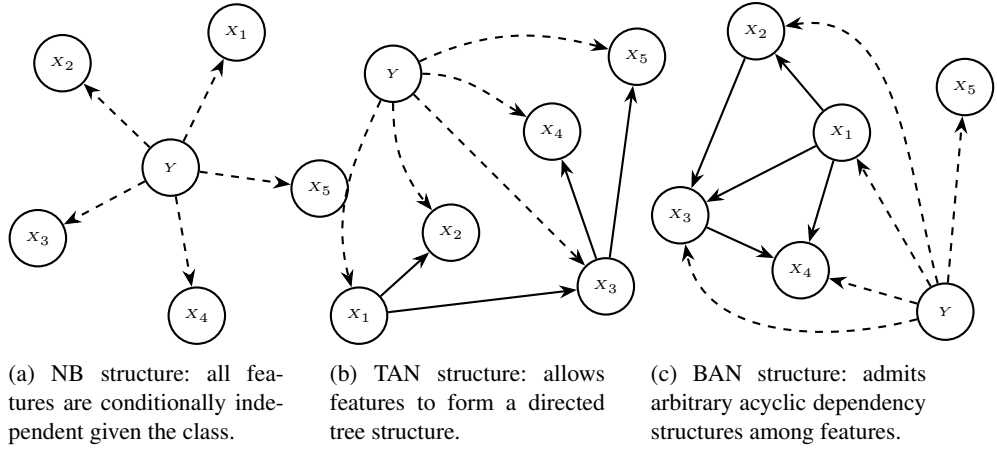
\begin{figure}[ht]
    \centering
    \begin{subfigure}[t]{0.3\textwidth}
        \centering
        \begin{tikzpicture}
            \node[node] (y) at (0,0) {$Y$};
            \node[node] (x1) at (64:2) {$X_1$};
            \node[node] (x2) at (136:2) {$X_2$};
            \node[node] (x3) at (208:2) {$X_3$};
            \node[node] (x4) at (280:2) {$X_4$};
            \node[node] (x5) at (352:2) {$X_5$};
            \draw[edge, dashed] (y) -- (x1);
            \draw[edge, dashed] (y) -- (x2);
            \draw[edge, dashed] (y) -- (x3);
            \draw[edge, dashed] (y) -- (x4);
            \draw[edge, dashed] (y) -- (x5);
        \end{tikzpicture}
        \caption{NB structure: all features are conditionally independent given the class.}
        \label{fig:str-nb}
    \end{subfigure}
    \hfill
    \begin{subfigure}[t]{0.3\textwidth}
        \centering
        \begin{tikzpicture}
            \node[node] (y) at (0,0) {$Y$};
            \node[node] (x1) at (262:3.25) {$X_1$};
            \node[node] (x2) at (290:2.25) {$X_2$};
            \node[node] (x3) at (315:4) {$X_3$};
            \node[node] (x4) at (340:2.25) {$X_4$};
            \node[node] (x5) at (4:3.25) {$X_5$};
            \draw[edge, dashed] (y) to[out=245, in=105] (x1);
            \draw[edge, dashed] (y) to[out=270, in=130] (x2);
            \draw[edge, dashed] (y) -- (x3);
            \draw[edge, dashed] (y) to[out=0, in=150] (x4);
            \draw[edge, dashed] (y) to[out=25, in=165] (x5);
            \draw[edge] (x1) -- (x2);
            \draw[edge] (x1) -- (x3);
            \draw[edge] (x3) -- (x4);
            \draw[edge] (x3) -- (x5);
        \end{tikzpicture}
        \caption{TAN structure: allows features to form a directed tree structure.}
        \label{fig:str-tan}
    \end{subfigure}
    \hfill
    \begin{subfigure}[t]{0.3\textwidth}
        \centering
        \begin{tikzpicture}
            \node[node] (y) at (0,0) {$Y$};
            \node[node] (x1) at (120:2.75) {$X_1$};
            \node[node] (x2) at (124:4.5) {$X_2$};
            \node[node] (x3) at (160:3.75) {$X_3$};
            \node[node] (x4) at (164:2) {$X_4$};
            \node[node] (x5) at (85:3) {$X_5$};
            \draw[edge, dashed] (y) -- (x1);
            \draw[edge, dashed] (y) to[out=105, in=10] (x2);
            \draw[edge, dashed] (y) to[out=195, in=280] (x3);
            \draw[edge, dashed] (y) -- (x4);
            \draw[edge, dashed] (y) -- (x5);
            \draw[edge] (x1) -- (x2);
            \draw[edge] (x1) -- (x3);
            \draw[edge] (x1) -- (x4);
            \draw[edge] (x2) -- (x3);
            \draw[edge] (x3) -- (x4);
        \end{tikzpicture}
        \caption{BAN structure: admits arbitrary acyclic dependency structures among features.}
        \label{fig:str-ban}
    \end{subfigure}
    \caption{Representative structures of Bayesian network classifiers}
    \label{fig:main-comparison}
\end{figure}

\subsection{Incorporating Linear Gaussian Structure}

% Grzegorczyk2010, geiger1994

In many real-world applications, observed features are often continuous. The BNC can therefore be extended to a CGNC, in which all features are continuous and conditioned on a discrete class variable. This framework combines the structural flexibility of BNs with the analytical tractability of linear Gaussian models, thereby rendering it suitable for continuous-data classification \cite{friedman1998, perez2006}.

Formally, for each feature node $X_i$ and class $c \in \mathcal{Y}$, the conditional distribution given its parents follows a linear Gaussian form:
\begin{gather}
    X_i \mid \bm{X}_{\paset_i}, Y=c \sim \mathcal{N}(\hat{\mu}_{i\mid c}, \sigma_{i\mid c}^2) \eqcomma \\
    \hat{\mu}_{i\mid c} = \sum_{j \in \paset_i} w_{ij}^c X_j + b_i^c  \eqcomma
\end{gather}
where $\sigma_{i \mid c}^2$ is a fixed conditional variance and $\hat{\mu}_{i \mid c}$ is an affine function of the parent variables. Here, $w_{ij}^c$ denotes the regression coefficient of parent $X_j$ on node $X_i$, and $b_i^c$ is the intercept term. All parameters are class-dependent and estimated independently for each class.

To clarify how CGNCs handle continuous features in likelihood computation, we state the following property, which formalizes the conditional Gaussian approximation used for tractable inference.

\begin{property}[Conditional gaussian approximation]
    \label{prop:glc}
    In CGNCs, the conditional probability of each continuous variable given its parents and class label is approximated by evaluating the corresponding univariate Gaussian PDF at the observed point:
    \begin{equation*}
        \prob(x_i \mid \bm{x}_{\paset_i}, Y = c) = \frac{1}{\sqrt{2 \pi \sigma_{i\mid c}^2}} \exp \left( - \frac{{(x_i - \hat{\mu}_{i\mid c})}^2}{2 \sigma_{i\mid c}^2} \right) \eqperiod
    \end{equation*}
    This approximation avoids integration, enabling tractable inference over continuous domains.
\end{property}

This approximation is directly applied in the MAP decision rule of \cref{eq:bnc_classif}. It ensures consistent evaluation across continuous features and preserves the relative magnitudes of conditional probabilities across classes and feature nodes \cite{john1995}.

\section{Robust CE Framework}
\label{sec:methodology}

\subsection{Incorporating CGNC into the CE Framework}

In the context of a binary CGNC, where the classification target takes values from $\mathcal{Y} = \{0, 1\}$, let $\mathcal{X} \subset \mathbb{R}^n$ denote the input domain, and let $\tau \in (0, 1)$ be a threshold specifying the minimum posterior probability required for the desired classification outcome. Given a factual instance $\bm{x}^{\mathrm{fac}} \in \mathcal{X}$ that is initially classified as class $0$ with $\prob(Y = 1 \mid \bm{x}^{\mathrm{fac}}) < \tau$, the goal is to identify a nearby instance whose posterior probability for class $1$ exceeds $\tau$.

Accordingly, the baseline counterfactual search method is posed as the following optimization problem:
\begin{align}
    \label{eq:std_constr}
    \begin{split}
        & \argmin_{\bm{x} \in \mathcal{X}} \enspace d(\bm{x}^{\mathrm{fac}}, \bm{x}) \\
        & \subjto \enspace \prob(Y = 1 \mid \bm{x}) \geq \tau \eqcomma
    \end{split}
\end{align}
where $d: \mathcal{X} \times \mathcal{X} \to \mathbb{R}_{\geq 0}$ denotes a distance function that quantifies the deviation between instances. Its specific form and properties are discussed in \cref{mahalanobis}.

As derived from the law of total probability, the denominator in the posterior expression in \cref{eq:post_prob} takes the form
\begin{equation}
    \mathbb{P}(\bm{x}) = \sum_{c' \in \mathcal{Y}} \mathbb{P}(Y = c') \cdot \mathbb{P}(\bm{x} \mid Y = c') \eqcomma
\end{equation}
which serves as a scaling factor that depends only on $\bm{x}$ and is independent of any class label. It effectively acts as a normalization constant to ensure that the posterior probabilities sum to one. We define the unnormalized joint likelihood of class $c \in \mathcal{Y}$ as
\begin{equation}
    \label{eq:prop_clf}
    \begin{split}
        h_c(\bm{x}) & \coloneq \prob(Y = c) \cdot \prob(\bm{x} \mid Y=c) \\
        &= \prob(Y = c) \cdot \prod_{i=1}^{n} \prob(x_i \mid \bm{x}_{\paset_i}, Y=c) \eqperiod
    \end{split}
\end{equation}
Substituting this into the constraint in \cref{eq:std_constr}, we obtain
\begin{equation}
    \frac{h_1(\bm{x})}{h_0(\bm{x}) + h_1(\bm{x})} \geq \tau \Longrightarrow \frac{h_1(\bm{x})}{h_0(\bm{x})} \geq \frac{\tau}{1 - \tau} \eqperiod
\end{equation}

Due to the proportionality $\prob(Y = c \mid \bm{x}) \propto h_c(\bm{x})$, the likelihood ratio in the constraint preserves the relative ordering of the posterior probabilities, thereby eliminating the dependence on the evidence term $\prob(\bm{x})$. In practice, a common choice is $\tau = 0.5$, in which case the constraint reduces to $h_1(\bm{x}) \geq h_0(\bm{x})$. This reflects the MAP decision rule for binary classification as expressed in \cref{eq:bnc_classif}, requiring that the counterfactual instance be more likely assigned to class $1$ than to class $0$.

In high-dimensional settings, directly multiplying numerous small conditional densities can lead to \emph{arithmetic underflow} in finite-precision arithmetic. To address this, we first take the logarithm of $h_c(\bm{x})$, which converts products into sums:
\begin{equation}
    \label{eq:prop_logclf}
    \log h_c(\bm{x}) = \log \prob(Y = c) + \sum_{i=1}^{n} \log \prob(x_i \mid \bm{x}_{\paset_i}, Y = c) \eqperiod
\end{equation}
This log-domain representation avoids multiplicative instability and facilitates downstream optimization.

Building on this transformation, we respectively define the \emph{log-relative likelihood} and the corresponding log-threshold:
\begin{gather}
    H(\bm{x}) \coloneq \log \frac{h_1(\bm{x})}{h_0(\bm{x})} = \log h_1(\bm{x}) - \log h_0(\bm{x}) \eqcomma \\
    \tau' \coloneq \log \frac{\tau}{1 - \tau} \eqperiod
\end{gather}
The decision function $H: \mathcal{X} \to \mathbb{R}$ captures the difference between the unnormalized log-joint likelihoods of the two classes, and $\tau' \in \mathbb{R}$ denotes the logit-transformed decision threshold. Based on these definitions, the counterfactual search problem can be reformulated as
\begin{align}
    \begin{split}
        & \argmin_{\bm{x} \in \mathcal{X}} \enspace d(\bm{x}^{\mathrm{fac}}, \bm{x}) \\
        & \subjto \enspace H(\bm{x}) \geq \tau' \eqperiod
    \end{split}
\end{align}

\subsection{Distance Metric and Whitening Transformation}
\label{mahalanobis}

In the original data space, feature scales vary and statistical dependencies exist. Applying Euclidean distance directly ignores these properties, leading to distorted proximity estimates. While standardization addresses scale imbalance, it does not eliminate correlations. In contrast, \emph{whitening} transforms the data into a space where features are both uncorrelated and normalized, yielding a more faithful geometric representation.

In this study, counterfactual search begins from a factual instance classified as class~$0$. Under the CGNC framework, features are generated conditionally on the class label, implying that any counterfactual derived from a class-$0$ instance is still assumed to follow $\prob(\bm{X} \mid Y=0)$, regardless of its eventual predicted label. To remain consistent with this generative constraint, we construct the whitening transformation solely based on the statistical profile of class~$0$.

Notably, applying the class-$0$ whitening globally to the input data would distort the geometry of class~$1$ and require re-deriving its class-conditional parameters in the transformed space, without reducing the computational burden. By confining whitening to the distance metric, we preserve the original data geometry while enabling distribution-aware distance evaluation. Therefore, whitening is applied only within the metric.

Accordingly, we adopt a Mahalanobis-$\ell_p$ norm that measures distances in the whitened space \cite{kulis2013}. Starting from $\bm{\Sigma}_0$, the covariance matrix of class~$0$, we construct the whitening matrix $\bm{\Sigma}_0^{-1/2}$ to transform data based on its covariance structure. The distance between any two points $\bm{x}, \bm{x}' \in \mathcal{X}$ is defined as
\begin{equation}
    d(\bm{x}', \bm{x}) \coloneq {\lVert \bm{x} - \bm{x}' \rVert}_{p, \bm{\Sigma}^{-1/2}_0} = { \left\lVert \bm{\Sigma}^{-1/2}_0 (\bm{x} - \bm{x}') \right\rVert }_p  \eqperiod
\end{equation}
This norm retains the geometric essence of the classical Mahalanobis distance by weighting directions with respect to the covariance structure of class $0$ \cite{demaesschalck2000}.\footnote{When $p = 2$, this reduces to the classical Mahalanobis distance $d(\bm{x}', \bm{x}) = \sqrt{(\bm{x} - \bm{x}')^\top \bm{\Sigma}_0^{-1} (\bm{x} - \bm{x}')}$, commonly used in classical statistical inference.}

As illustrated in \cref{fig:mahalanobis_lp}, the whitening transformation maps Mahalanobis-$\ell_p$ balls in the original space to standard $\ell_p$ balls, with red points indicating samples from class $0$. This clearly reveals the alignment between the transformed geometry and the chosen distance metric.

\begin{figure}[ht]
    \centering
    \begin{subfigure}[t]{0.49\textwidth}
        \centering
        \includegraphics[scale=0.36]{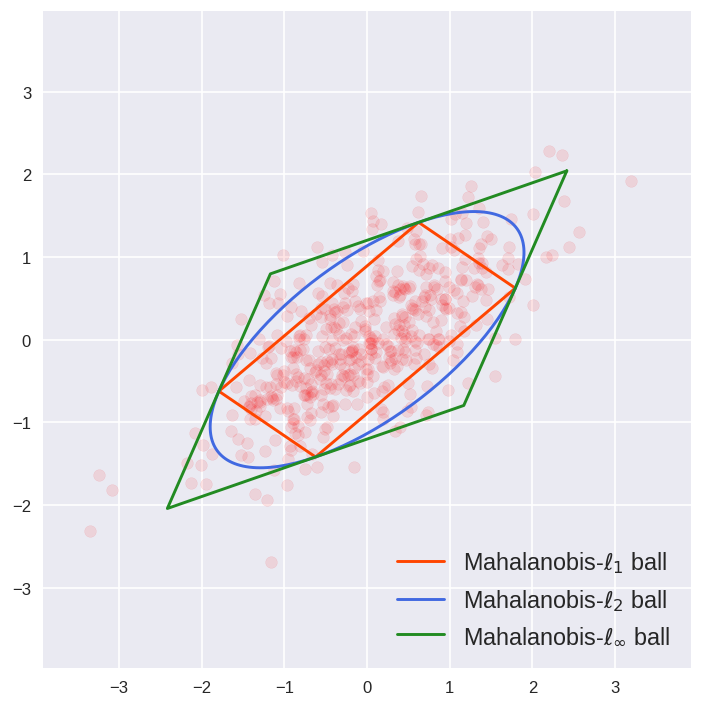}
        \caption{In the original space, the Mahalanobis-$\ell_p$ balls are distorted by feature scales and correlations.}
    \end{subfigure}
    \hfill
    \begin{subfigure}[t]{0.49\textwidth}
        \centering
        \includegraphics[scale=0.36]{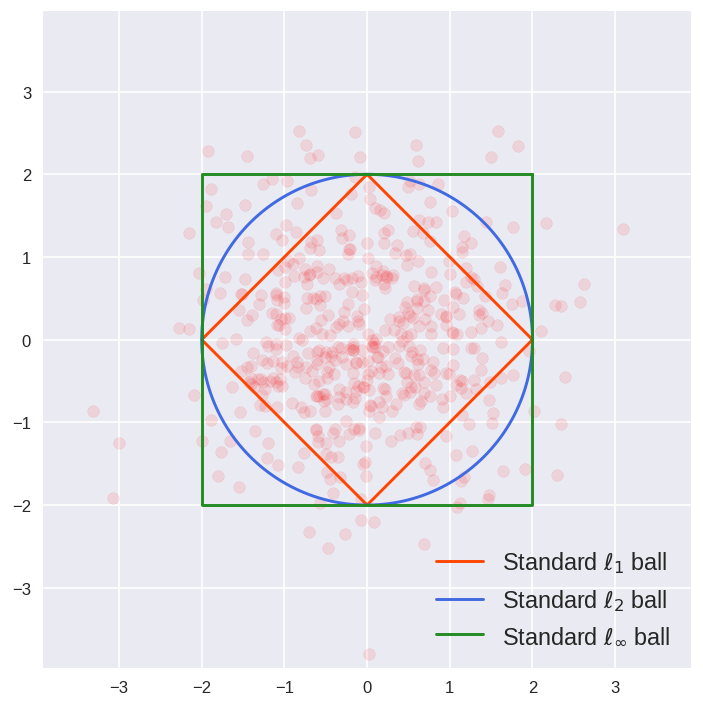}
        \caption{After whitening, these balls are transformed into standard $\ell_p$ balls.}
    \end{subfigure}
    \caption{Comparison of Mahalanobis-$\ell_p$ balls in the original and whitened spaces}
    \label{fig:mahalanobis_lp}
\end{figure}

Since the principal inverse square root of a matrix is not unique, we compute the whitening matrix for the Mahalanobis-$\ell_p$ distance using the Cholesky decomposition of $\bm{\Sigma}_0^{-1}$. This produces a unique lower triangular factor with positive diagonal entries, ensuring numerical stability and preserving the original variable ordering, both essential when the whitening transform is embedded in a distance metric. Compared to principal component analysis (PCA) or zero-phase component analysis (ZCA) whitening, the Cholesky approach avoids rotational ambiguity and keeps the Mahalanobis metric aligned with the original feature axes, which is particularly advantageous for interpretability tasks \cite{kessy2018}.

\subsection{Robustness in Counterfactual Search}

In real-world deployments, CE is expected to remain stable under small and plausible implementation errors. To account for such perturbations, we incorporate robustness into the optimization process by requiring the classification constraint to hold for all admissible perturbations within an uncertainty set \cite{bental2009}.

Let $\bm{\delta} \in \mathbb{R}^n$ denote a perturbation applied to a counterfactual instance. We define the uncertainty set $\mathcal{U}_\gamma$, parameterized by a robustness budget $\gamma > 0$, as
\begin{equation}
    \mathcal{U}_\gamma \coloneq \left\{\bm{\delta} \in \mathbb{R}^n \;\middle|\; {\lVert \bm{\delta} \rVert}_{p, \bm{\Sigma}^{-1/2}_0} \leq \gamma \right\} \eqcomma
\end{equation}
where the norm is defined in the whitened space of class~$0$, ensuring alignment with the geometry introduced in \cref{mahalanobis}. This formulation enforces robustness through additive perturbations bounded by a structured, distribution-aware region. Different choices of norms induce different geometric properties in the uncertainty set.

With this formulation, the optimization problem with robustness to recourse requires the classification constraint to hold for all perturbed instances within the uncertainty set $\mathcal{U}_\gamma$. The optimization problem is then defined as
\begin{align}
    \label{eq:robust_opt}
    \begin{split}
        & \argmin_{\bm{x} \in \mathcal{X}} \enspace d(\bm{x}^{\mathrm{fac}}, \bm{x}) \\
        & \subjto \enspace H(\bm{x} + \bm{\delta}) \geq \tau' \enspace \forall \bm{\delta} \in \mathcal{U}_\gamma \eqperiod
    \end{split}
\end{align}

Consequently, enforcing the constraint for every perturbation in $\mathcal{U}_\gamma$ yields a semi-infinite program, where the decision variable is finite-dimensional but the inequality constraints are infinitely many. This makes the problem intractable in its original form. We adopt the cutting-set procedure, which reformulates the robust constraint into an iterative process alternating between solving a restricted problem over a finite scenario set and identifying new worst-case perturbations that may violate the constraint \cite{mutapcic2009}.

At each iteration, the \emph{master problem} (MP) solves
\begin{align}
    \begin{split}
        & \argmin_{\bm{x} \in \mathcal{X}} \enspace d(\bm{x}^{\mathrm{fac}}, \bm{x}) \\
        & \subjto \enspace H(\bm{x} + \hat{\bm{\delta}}) \geq \tau' \enspace \forall \hat{\bm{\delta}} \in \mathcal{S} \eqcomma
    \end{split}
\end{align}
where $\mathcal{S} \subset \mathcal{U}_\gamma$ is a finite scenario set of previously identified perturbation scenarios. Initially, we let $\mathcal{S} = \{ \bm{0} \}$, so that the first MP considers only the unperturbed case. Let $\hat{\bm{x}}$ denote the current solution to the MP, then the \emph{adversarial problem} (AP) evaluates
\begin{equation}
    \argmax_{\bm{\delta} \in \mathcal{U}_\gamma} \enspace [\tau' - H(\hat{\bm{x}} + \bm{\delta})] \eqcomma
\end{equation}
which identifies a perturbation $\hat{\bm{\delta}}$ that most violates the robustness condition. If a violation is found, it is added to $\mathcal{S}$, and the MP is solved again with the extended set.

This iterative procedure incrementally refines the constraint set while keeping each MP tractable. As illustrated by the example in \cref{fig:robustness}, the resulting solution $\hat{\bm{x}}^*$ satisfies the robustness condition over the growing scenario set $\mathcal{S}$, a discrete approximation of the original uncertainty set $\mathcal{U}_\gamma$. This decomposition replaces the original formulation with a sequence of finite, solvable subproblems, without significantly compromising robustness.

\begin{figure}[ht]
    \centering
    \includegraphics[scale=0.36]{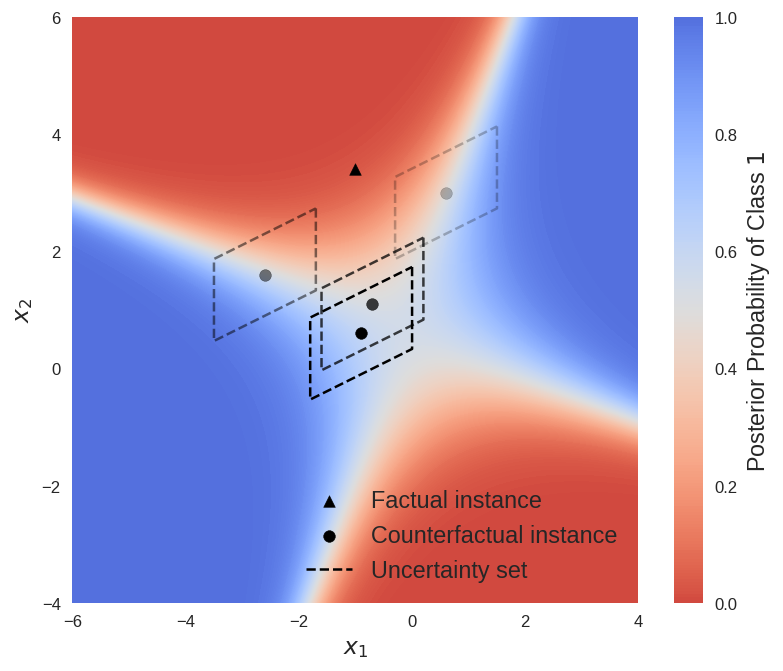}
    \caption{Iterative refinement of uncertainty sets in robust counterfactual search}
    \label{fig:robustness}
\end{figure}

\section{Algorithms}
\label{sec:problemformulation}

\subsection{Decision Function Expansion and Algorithm Design}

To reformulate the MP and AP as tractable optimization models, we begin by expanding the log-relative likelihood structure of the decision function $H$, evaluated at perturbed inputs $\bm{x} + \bm{\delta}$. This yields a quadratic form arising from the squared deviations of conditional Gaussian densities, involving both quadratic terms such as $x_i^2$ and $\delta_i^2$, and bilinear terms such as $x_j x_k$ and $\delta_j \delta_k$ for $j \neq k$. We summarize the key structural components relevant to modeling below. 

For each feature node $X_i$, we define
\begin{equation}
    \label{eq:indexset}
    \paset_i^+ \coloneq \{i\} \cup \paset_i \eqcomma
\end{equation}
and, for each class $c$, introduce the coefficient
\begin{equation}
    \label{eq:coeff}
    a_{ij}^c \coloneq
    \begin{cases}
        1, & j = i, \\
        -w_{ij}^c, & j \in \paset_i
    \end{cases} \eqperiod
\end{equation}
With the conditional mean expressed as a linear combination of the parent nodes, the set $\paset_i^+$ contains the variable indices involved in the conditional deviation, while the coefficients $a_{ij}^c$ specify the corresponding linear relationships between a node and its parents.

A full algebraic derivation of the function expansion is provided in \cref{appx:quad_expand}. Based on these results, and replacing the class prior $\rho_c \equiv \prob(Y = c)$ with symbolic notation, the expanded formulation of the MP is
\begin{align}
    \begin{split}
        & \! \argmin_{\bm{x} \in \mathcal{X}} \enspace d(\bm{x}^{\mathrm{fac}}, \bm{x}) \\
        & \begin{aligned}
            \subjto \enspace & \sum_{c \in \mathcal{Y}} (2c-1) \left[ \log \rho_c - \sum_{i=1}^{n} \left( \log \sigma_{i\mid c} + \frac{D_i^c(\bm{x}; \hat{\bm{\delta}})}{2 \sigma_{i\mid c}^2} \right) \right] \geq \tau' \enspace \forall \hat{\bm{\delta}} \in \mathcal{S}, \\
            & D_i^c(\bm{x}; \hat{\bm{\delta}}) = \sum_{\mathclap{\substack{j \in \paset_i^+ \\ k \in \paset_i^+}}} a_{ij}^c a_{ik}^c x_j x_k + 2 \xi_i^c(\hat{\bm{\delta}}) \sum_{\mathclap{j \in \paset_i^+}} a_{ij}^c x_j + {\xi_i^c(\hat{\bm{\delta}})}^2 \enspace \forall c \in \mathcal{Y}, i \in [n], \hat{\bm{\delta}} \in \mathcal{S}, \\
            & \xi_i^c(\hat{\bm{\delta}}) = \sum_{\mathclap{k \in \paset_i^+}} a_{ik}^c \hat{\delta}_k - b_i^c \enspace \forall c \in \mathcal{Y}, i \in [n], \hat{\bm{\delta}} \in \mathcal{S} \eqperiod
        \end{aligned}
    \end{split}
\end{align}
Similarly, the AP is
\begin{align}
    \begin{split}
        & \! \argmax_{\bm{\delta} \in \mathcal{U}_\gamma} \enspace \tau' - \sum_{c \in \mathcal{Y}} (2c-1) \left[ \log \rho_c - \sum_{i=1}^{n} \left( \log \sigma_{i\mid c} + \frac{D_i^c(\bm{\delta}; \hat{\bm{x}})}{2 \sigma_{i\mid c}^2} \right) \right] \\
        & \begin{aligned}
            \subjto \enspace & D_i^c(\bm{\delta}; \hat{\bm{x}}) = \sum_{\mathclap{\substack{j \in \paset_i^+ \\ k \in \paset_i^+}}} a_{ij}^c a_{ik}^c \delta_j \delta_k + 2 \xi_i^c(\hat{\bm{x}}) \sum_{\mathclap{j \in \paset_i^+}} a_{ij}^c \delta_j + {\xi_i^c(\hat{\bm{x}})}^2 \enspace \forall c \in \mathcal{Y}, i \in [n], \\
            & \xi_i^c(\hat{\bm{x}}) = \sum_{\mathclap{k \in \paset_i^+}} a_{ik}^c \hat{x}_k - b_i^c \enspace \forall c \in \mathcal{Y}, i \in [n] \eqperiod
        \end{aligned}
    \end{split}
\end{align}

To ensure finite termination of the counterfactual search algorithm, the decision function $H$ must be Lipschitz continuous over a compact domain $\mathcal{X}$. Such regularity guarantees that the cutting-set procedure reduces the violation of the robustness condition $\varphi$ below any tolerance level $\varepsilon > 0$ in finitely many iterations \cite{mutapcic2009}. The following theorem formalizes this property, with the full proof in \cref{appx:conv}.

\begin{theorem}[Finite-termination guarantee]
    \label{lemma1}
    Assume that the input domain $\mathcal{X} \subset \mathbb{R}^n$ is compact, and that the classifier $H: \mathcal{X} \to \mathbb{R}$ is Lipschitz continuous with a constant $G > 0$, meaning that for any $\bm{x}, \bm{x}' \in \mathcal{X}$,
    \begin{equation*}
        |H(\bm{x}) - H(\bm{x}')| \leq G \cdot \lVert \bm{x} - \bm{x}' \rVert \eqcomma
    \end{equation*}
    where $\lVert \cdot \rVert$ denotes the primary norm used in the optimization domain. Then, for any tolerance level $\varepsilon > 0$, the cutting-set procedure terminates in a finite number of iterations. The final solution $\hat{\bm{x}}^* \in \mathcal{X}$ is guaranteed to satisfy
    \begin{equation*}
        \max_{\bm{\delta} \in \mathcal{U}_\gamma} \; [\tau' - H(\hat{\bm{x}}^* + \bm{\delta})] \leq \varepsilon \eqperiod
    \end{equation*}
\end{theorem}

\begin{proof}[Theorem~\ref{lemma1}]
    \label{appx:conv}

    We adopt the cutting-set framework for solving robust optimization problems via iterative scenario refinement. The convergence analysis provides an upper bound on the number of iterations required, assuming Lipschitz continuity of the constraint functions and boundedness of the feasible region \cite{mutapcic2009}.
    
    We aim to prove that the log-relative likelihood function $H(\bm{x})$ is Lipschitz continuous over the input domain $\mathcal{X}$, where distances are measured under a general norm $\lVert \cdot \rVert$. In our implementation, this norm corresponds to the Mahalanobis-$\ell_p$ norm, as specified in \cref{mahalanobis}.
    
    For any $\bm{x}, \bm{x}' \in \mathcal{X}$, differentiability implies:
    \begin{equation}
        H(\bm{x}) - H(\bm{x}') = {\nabla H(\bm{x})}^\top (\bm{x} - \bm{x}')
    \end{equation}
    
    Applying H\"older's inequality, we obtain the standard Lipschitz condition:
    \begin{equation}
        | H(\bm{x}) - H(\bm{x}') | \leq {\lVert \nabla H(\bm{x}) \rVert}_* \cdot \lVert \bm{x} - \bm{x}' \rVert
    \end{equation}
    where $\lVert \cdot \rVert$ denotes the primary norm used to define the geometry of the input space, and ${\lVert \cdot \rVert}_*$ is its corresponding dual norm. This dual norm structure reflects the standard pairing between function variations and input geometry in Lipschitz functional spaces \cite{weaver2018}. To ensure boundedness of the gradient over the domain, we assume that $\mathcal{X}$ is compact and fully contained within a norm ball of radius $R$, that is,
    
    \begin{equation}
        \mathcal{X} \subseteq \{ \bm{x} \in \mathbb{R}^n: \lVert \bm{x} \rVert \leq R \}
    \end{equation}
    
    We now proceed to derive an explicit upper bound on $G$ under the structural assumptions of the CGNC model.
    
    We first observe that
    \begin{equation}
        \nabla H(\bm{x}) = \nabla \log h_1(\bm{x}) - \nabla \log h_0(\bm{x})
    \end{equation}
    By the triangle inequality, it follows that
    \begin{equation}
        {\lVert \nabla H(\bm{x}) \rVert}_* \leq {\lVert \nabla \log h_1(\bm{x}) \rVert}_* + {\lVert \nabla \log h_0(\bm{x}) \rVert}_*
    \end{equation}
    
    Thus, the problem reduces to bounding the gradient norm of the log-joint probability under the CGNC for each class $c \in \mathcal{Y}$.
    \begin{align}
        \begin{split}
            \nabla \log h_c(\bm{x}) &= \nabla \left( - \sum_{i=1}^n \frac{{\left({\bm{a}_i^c}^\top \bm{x} - b_i^c \right)}^2}{2\sigma_{i\mid c}^2}  + C_c \right) \\
            &= - \sum_{i=1}^n \frac{{\bm{a}_i^c}^\top \bm{x} - b_i^c}{\sigma_{i\mid c}^2} \bm{a}_i^c
        \end{split}
    \end{align}
    
    This reflects the general result that the Lipschitz constant of the log-joint probability in a GBN is upper-bounded by the sum of the Lipschitz constants of its local conditional components \cite{honorio2011}.
    
    Taking the norm of the gradient expression, we use the H\"older's inequality to bound $\left| {\bm{a}_i^c}^\top \bm{x} \right| \leq {\lVert \bm{a}_i^c \rVert}_* \cdot \lVert \bm{x} \rVert \leq {\lVert \bm{a}_i^c \rVert}_* R$, and obtain
    \begin{equation}
        \begin{split}
            {\lVert \nabla \log h_c(\bm{x}) \rVert}_* & \leq \sum_{i=1}^n \frac{\left| {\bm{a}_i^c}^\top \bm{x} - b_i^c \right|}{\sigma_{i\mid c}^2} \cdot {\lVert \bm{a}_i^c \rVert}_* \\
            & \leq \sum_{i=1}^n \frac{{\lVert \bm{a}_i^c \rVert}_* R + |b_i^c|}{\sigma_{i\mid c}^2} \cdot {\lVert \bm{a}_i^c \rVert}_*
        \end{split}
    \end{equation}
    
    Finally, summing over all nodes and both classes yields an upper bound on the gradient norm of $H(\bm{x})$. Under this assumption, the Lipschitz constant of $H$ with respect to the norm $\lVert \cdot \rVert$ is given by
    \begin{align}
        \begin{split}
            G & \coloneq \sup_{\bm{x} \in \mathcal{X}} {\lVert \nabla H(\bm{x}) \rVert}_* \\
            & = \sum_{c \in \mathcal{Y}} \sum_{i=1}^n \frac{{\lVert \bm{a}_i^c \rVert}_* R + |b_i^c|}{\sigma_{i\mid c}^2} \cdot {\lVert \bm{a}_i^c \rVert}_*
        \end{split}
    \end{align}
    
    The Lipschitz continuity of $H$ ensures that any two solutions found by the MP must be separated by at least $\varepsilon / G$ in norm. Consequently, each new point must lie outside a ball of diameter $\varepsilon / G$ centered at any previous iterate. The number of such disjoint balls that can be packed into the feasible region, whose radius is $R$, leads to the following upper bound on the total number of iterations:
    \begin{equation}
    T \leq \left( \frac{RG}{\varepsilon} \right)^n
    \end{equation}
    which follows from a ball-packing argument \cite{mutapcic2009}.
    \qed
\end{proof}

With this guarantee, \cref{alg:opt_qp} presents the concrete implementation of the counterfactual search procedure, which alternates between MP and AP.

\begin{algorithm}
    \caption{Cutting-set procedure for robust counterfactual search}
    \label{alg:opt_qp}
    \begin{algorithmic}[1]
        \Require{$\bm{x}^{\mathrm{fac}} \in \mathcal{X}$, $\gamma > 0$, $\varepsilon > 0$}
        \Ensure{$\hat{\bm{x}}^*$}
        \State $\mathcal{S} \gets \{\bm{0}\}$
        \Repeat
            \State $\hat{\bm{x}} \gets$ solve MP with $\bm{x}^{\mathrm{fac}}$ and $\mathcal{S}$
            \State $\hat{\bm{\delta}} \gets$ solve AP with $\hat{\bm{x}}$ and $\mathcal{U}_\gamma$
            \State $\varphi \gets$ objective value of AP
            \State $\mathcal{S} \gets \mathcal{S} \cup \{ \hat{\bm{\delta}} \}$
        \Until{$\varphi \leq \varepsilon$}
        \State $\hat{\bm{x}}^* \gets \hat{\bm{x}}$
        \State \Return{$\hat{\bm{x}}^*$}
    \end{algorithmic}
\end{algorithm}

\subsection{Quadratic Structure to MILP Reformulation}
\label{appx:milp}

Due to the presence of numerous second-order terms in the expanded formulation, both the MP and the AP take the form of quadratically constrained quadratic programs (QCQPs), which are specifically nonconvex in this setting.\footnote{The AP maximizes violation over an uncertainty set. When the set is linearly bounded, such as in Mahalanobis-$\ell_1$ or Mahalanobis-$\ell_\infty$ balls, the problem reduces to a quadratic program (QP), since only the objective is quadratic. As QP is a special case of QCQP, we still refer to both MP and AP as QCQPs.} The derivation of this nonconvexity is provided in \cref{appx:quad_nonconvex}. These problems are not only difficult to solve directly but also incompatible with standard MILP frameworks, precluding the use of global optimality bounds and convergence guarantees offered by MILP solvers.

For the purpose of reformulating the problem within the MILP framework, we transform the second-order terms into linear form. Specifically, we apply McCormick convex envelope relaxations, which approximate each second-order term using a set of linear inequalities derived from prescribed variable bounds \cite{mccormick1976}.

To systematically handle all second-order terms, we define the following index set of variable pairs:
\begin{equation}
    \mathcal{Q} \coloneq \bigcup_{i=1}^{n} \{ (j, k) \in \paset_i^+ \times \paset_i^+ \mid j \leq k \} \eqperiod
\end{equation}
Only index pairs with $j \leq k$ are retained to avoid symmetric duplicates. This set enumerates all unique variable pairs whose product appears as a second-order term in the expanded formulation, each of which will be approximated via McCormick relaxation.

However, the conventional McCormick method relaxes second-order terms solely based on global bounds, often resulting in excessively large feasible regions and degraded approximation quality. To mitigate this issue, we adopt a piecewise McCormick relaxation. For each $j \in [n]$, the feasible range of $x_j$ is given as $[\underline{x}_j, \overline{x}_j]$, which is then divided into $m \in \mathbb{N}$ subintervals, and a local linear relaxation is constructed within each subinterval \cite{bergamini2005, karuppiah2006}. The breakpoints are denoted as $\underline{x}_j = l_j^1 < l_j^2 < \dots < l_j^{m + 1} = \overline{x}_j$. For implementation purposes, we adopt uniform partitioning across all variables to ensure simplicity and consistency \cite{hasan2010}.

For each $j \in [n]$, we introduce binary variables $\lambda_j^r \in \{0, 1\}$ indicating whether $x_j$ falls into the $r$-th subinterval $[l_j^r, l_j^{r+1}]$. To ensure that exactly one segment is active for each variable, the constraint $\sum_{r=1}^m \lambda_j^r = 1$ is enforced. Under this partitioning structure, each relaxation variable $z_{jk} \approx x_j x_k$ is approximated using McCormick envelope constraints within its corresponding segment, resulting in a formulation that can be represented within the MILP framework.

Applying the above relaxation mechanism to both optimization problems, we first consider the MP, where the second-order terms $x_j x_k$ are approximated using the McCormick relaxation as defined above. This yields the following relaxed form of the MP:
\begin{align}
    \begin{split}
        & \! \argmin_{\bm{x} \in \mathcal{X}} \enspace d(\bm{x}^{\mathrm{fac}}, \bm{x}) \\
        & \begin{aligned}
            \subjto \enspace & \sum_{c \in \mathcal{Y}} (2c-1) \left[ \log \rho_c - \sum_{i=1}^{n} \left( \log \sigma_{i\mid c} + \frac{D_i^c(\bm{x} |_{\bm{z}}; \hat{\bm{\delta}})}{2 \sigma_{i\mid c}^2} \right) \right] \geq \tau' \enspace \forall \hat{\bm{\delta}} \in \mathcal{S}, \\
            & D_i^c(\bm{x} |_{\bm{z}}; \hat{\bm{\delta}}) = \sum_{\mathclap{\substack{j \in \paset_i^+ \\ k \in \paset_i^+}}} a_{ij}^c a_{ik}^c z_{jk} + 2 \xi_i^c(\hat{\bm{\delta}}) \sum_{\mathclap{j \in \paset_i^+}} a_{ij}^c x_j + {\xi_i^c(\hat{\bm{\delta}})}^2 \enspace \forall c \in \mathcal{Y}, i \in [n], \hat{\bm{\delta}} \in \mathcal{S}, \\
            & \xi_i^c(\hat{\bm{\delta}}) = \sum_{\mathclap{k \in \paset_i^+}} a_{ik}^c \hat{\delta}_k - b_i^c \enspace \forall c \in \mathcal{Y}, i \in [n], \hat{\bm{\delta}} \in \mathcal{S}, \\
            & \bigvee_{r=1}^{m} \left[ \begin{array}{c}
                \lambda_j^r, \\
                l_j^r \leq x_j \leq l_j^{r + 1}, \\
                \begin{rcases}
                    z_{jk} \geq x_j l_k^1 + l_j^r x_k - l_j^r l_k^1, \\
                    z_{jk} \geq x_j l_k^{m + 1} + l_j^{r + 1} x_k - l_j^{r + 1} l_k^{m + 1}, \\
                    z_{jk} \leq x_j l_k^{m + 1} + l_j^r x_k - l_j^r l_k^{m + 1}, \\
                    z_{jk} \leq x_j l_k^1 + l_j^{r + 1} x_k - l_j^{r + 1} l_k^1
                \end{rcases} \enspace \forall k \in \{k \mid (j, k) \in \mathcal{Q}\}
            \end{array} \right] \enspace \forall j \in [n], \\
            & z_{kj} = z_{jk} \enspace \forall (j, k) \in \mathcal{Q}, j \neq k, \\
            & l_j^r = \underline{x}_j + \frac{r - 1}{m} (\overline{x}_j - \underline{x}_j) \enspace \forall j \in [n], r \in [m + 1], \\
            & \lambda_j^r \in \{0, 1\} \enspace \forall j \in [n], r \in [m], \\
            & \sum_{r=1}^{m} \lambda_j^r = 1 \enspace \forall j \in [n] \eqperiod
        \end{aligned}
    \end{split}
\end{align}
For the AP, the optimization variables are the perturbation terms $\bm{\delta}$. We adopt the same relaxation strategy for all second-order terms $\delta_j \delta_k$, introducing relaxation variables $\eta_{jk} \approx \delta_j \delta_k$, binary indicators $\lambda_j^r$, and bounds $[\underline{\delta}_j, \overline{\delta}_j]$, following the same definitions as in the MP. This yields the following relaxed form of the AP:
\begin{align}
    \begin{split}
        & \! \argmax_{\bm{\delta} \in \mathcal{U}_\gamma} \enspace \tau' - \sum_{c \in \mathcal{Y}} (2c-1) \left[ \log \rho_c - \sum_{i=1}^{n} \left( \log \sigma_{i\mid c} + \frac{D_i^c(\bm{\delta} |_{\bm{\eta}}; \hat{\bm{x}})}{2 \sigma_{i\mid c}^2} \right) \right] \\
        & \begin{aligned}
            \subjto \enspace  & D_i^c(\bm{\delta} |_{\bm{\eta}}; \hat{\bm{x}}) = \sum_{\mathclap{\substack{j \in \paset_i^+ \\ k \in \paset_i^+}}} a_{ij}^c a_{ik}^c \eta_{jk} + 2 \xi_i^c(\hat{\bm{x}}) \sum_{\mathclap{j \in \paset_i^+}} a_{ij}^c \delta_j + {\xi_i^c(\hat{\bm{x}})}^2 \enspace \forall c \in \mathcal{Y}, i \in [n], \\
            & \xi_i^c(\hat{\bm{x}}) = \sum_{\mathclap{k \in \paset_i^+}} a_{ik}^c \hat{x}_k - b_i^c \enspace \forall c \in \mathcal{Y}, i \in [n], \\
            & \bigvee_{r=1}^{m} \left[ \begin{array}{c}
                \lambda_j^r, \\
                l_j^r \leq \delta_j \leq l_j^{r + 1}, \\
                \begin{rcases}
                    \eta_{jk} \geq \delta_j l_k^1 + l_j^r \delta_k - l_j^r l_k^1, \\
                    \eta_{jk} \geq \delta_j l_k^{m + 1} + l_j^{r + 1} \delta_k - l_j^{r + 1} l_k^{m + 1}, \\
                    \eta_{jk} \leq \delta_j l_k^{m + 1} + l_j^r \delta_k - l_j^r l_k^{m + 1}, \\
                    \eta_{jk} \leq \delta_j l_k^1 + l_j^{r + 1} \delta_k - l_j^{r + 1} l_k^1
                \end{rcases} \enspace \forall k \in \{k \mid (j, k) \in \mathcal{Q}\}
            \end{array} \right] \enspace \forall j \in [n], \\
            & \eta_{kj} = \eta_{jk} \enspace \forall (j, k) \in \mathcal{Q}, j \neq k, \\
            & l_j^r = \underline{\delta}_j + \frac{r - 1}{m} (\overline{\delta}_j - \underline{\delta}_j) \enspace \forall j \in [n], r \in [m + 1], \\
            & \lambda_j^r \in \{0, 1\} \enspace \forall j \in [n], r \in [m], \\
            & \sum_{r=1}^{m} \lambda_j^r = 1 \enspace \forall j \in [n] \eqperiod
        \end{aligned}
    \end{split}
\end{align}

Although piecewise McCormick relaxation enables reformulating the original nonlinear problem as a MILP, its accuracy is highly sensitive to the prescribed variable bounds and partition count. Loose bounds or too few partitions enlarge the relaxed feasible region, admitting many solutions that are infeasible in the original problem space. Conversely, excessive partitions inflate the MILP size, leading to a substantial computational burden.

To avoid these issues, we incorporate a \emph{bound tightening} mechanism within the iterative procedure. At the $t$-th iteration of MP, the search domain of each variable is updated to focus around the previous solution $\hat{\bm{x}}^{(t-1)}$, thereby progressively restricting the relaxed region to the most relevant area. Let $[\underline{x}_j^\init, \overline{x}_j^\init]$ be the initial bounds of $x_j$, and let $\nu \in (0, 1)$ denote the contraction factor. The bounds for the $t$-th iteration are computed as
\begin{gather}
    \underline{x}_j^{(t)} =
    \begin{cases}
        \underline{x}_j^\init, & t = 1, \\
        \nu^{t - 1} \cdot \underline{x}_j^\init + (1 - \nu^{t - 1}) \cdot \hat{x}_j^{(t - 1)}, & t > 1
    \end{cases} \eqcomma \\
    \overline{x}_j^{(t)} =
    \begin{cases}
        \overline{x}_j^\init, & t = 1, \\
        \nu^{t - 1} \cdot \overline{x}_j^\init + (1 - \nu^{t - 1}) \cdot \hat{x}_j^{(t - 1)}, & t > 1
    \end{cases} \eqperiod
\end{gather}
In parallel, starting from an initial number of partitions $m^\init$, the partition count for the $t$-th iteration is reduced according to
\begin{equation}
    m^{(t)} = \lceil \nu^{t - 1} \cdot m^\init \rceil \eqperiod
\end{equation}
This strategy allows the relaxed region to progressively align with the original nonlinear structure, while effectively containing the computational cost induced by model growth \cite{deng2021}.

In AP, each variable $\delta_j$ for $j \in [n]$ is bounded by
\begin{equation}
    [\underline{\delta}_j, \overline{\delta}_j] = \left[ \min_{\bm{\delta} \in \mathcal{U}_\gamma} \delta_j, \max_{\bm{\delta} \in \mathcal{U}_\gamma} \delta_j \right] \eqcomma
\end{equation}
which corresponds to the projection of the uncertainty set $\mathcal{U}_\gamma$ onto the $j$-th coordinate. These bounds remain fixed to ensure that the computed maximal violation covers all possible perturbations. Additionally, to preserve consistent approximation quality across iterations, the AP uses a fixed partition count $m^\init$, maintaining stable relaxation accuracy.

From an intuitive standpoint, applying piecewise McCormick relaxations alters the geometry of the feasible region. As a result, even though the original decision function $H$ is Lipschitz continuous and theoretically guarantees the convergence of the cutting-set procedure, this guarantee no longer holds. Nevertheless, bound tightening progressively refines the relaxed model to better resemble the original nonlinear structure. We therefore adopt the same tolerance level $\varepsilon$ from \cref{lemma1} as the termination criterion.

\section{Experiments}
\label{sec:experiments}

\subsection{General Experimental Settings}

All experiments were conducted using Python 3.12.7 (CPython) on a 12th Gen Intel Core i7-12700 processor (20 cores, 20 threads) running Ubuntu 20.04 LTS. Both the piecewise McCormick relaxation-based MILP and the original nonconvex QCQP formulations were solved with Gurobi Optimizer 12.0.1 via the official gurobipy interface \cite{gurobi2024}. The solver was configured to utilize all available threads, with a MIP gap tolerance of 0.01, and all other parameters were left at their default settings.

In our setup, we use the Mahalanobis-$\ell_\infty$ distance, which imposes per-dimension perturbation bounds expressible as linear constraints. The classification threshold is fixed at the standard $\tau = 0.5$, and the tolerance is uniformly set to $\varepsilon = 0.001$.

For the MILP-based approach, the initial bounds $[\underline{x}_j^\mathrm{init}, \overline{x}_j^\mathrm{init}]$ for $j \in [n]$ are determined from the 5th and 95th percentiles of the corresponding feature in the dataset, ensuring that the search space covers the typical data range. The initial number of partitions is set to $m^\mathrm{init} = 20$, with a contraction factor of $\nu = 0.5$.

\subsection{Empirical Evaluation on Benchmark Datasets}

This study evaluates the proposed method on three real-world datasets: Banknote Authentication, Pima Indians Diabetes, and Ionosphere \cite{kelly2025}. To construct CGNCs, we first learn a directed structure among continuous features, while the discrete class node is connected to all features, consistent with the CGNC framework.

Under the NB, all features are assumed conditionally independent given the class, requiring no structure learning. For TAN, each feature is discretized using equal-frequency binning to estimate conditional mutual information under each class. These estimates are aggregated using class priors to construct a global mutual information matrix, from which a maximum spanning tree is derived and oriented into a DAG \cite{cheng1999}. For BAN, we apply the NOTEARS algorithm, which formulates structure learning as a continuous optimization problem under an acyclicity constraint, resulting in a sparse DAG  \cite{zheng2018}.

Given the learned structure, we fit class-conditional linear Gaussian models for each node by regressing on its parents. For nodes without parents, class-specific sample means and variances are used directly. This constructs a fully specified CGNC that captures the dependencies among features and serves as the foundation for counterfactual search and robustness evaluation.

We conducted experiments under two robustness budget values, $\gamma = 0.01$ and $\gamma = 0.05$, respectively. For each configuration, 25 independent runs were performed across different structure types and both formulation types.

\cref{tab:bugdet001} summarizes performance under a robustness budget of $\gamma = 0.01$, reporting runtime and iteration count with standard errors, as well as \emph{early stop} frequency and timeouts. Any run exceeding 3600 seconds is considered a timeout and excluded from the computation of these statistics. NB serves as a conditional independence baseline, enabling assessment of computational and robustness effects of feature dependencies through TAN and BAN.

\begin{table}[ht]
    \centering
    \begin{threeparttable}
        \tiny
        \caption{Results under robustness budget $\gamma = 0.01$}
        \begin{tabular}{ccccccccccc}
            \toprule
            \multicolumn{3}{c}{} & \multicolumn{4}{c}{Piecewise McCormick Relaxation-Based MILP} & \multicolumn{4}{c}{Original Nonconvex QCQP} \\
            \cmidrule(lr){4-7} \cmidrule(lr){8-11}
            \# Nodes & Structure & \# Edges & Runtime (s) & \# Iterations & \# Early Stops & \# Timeouts & Runtime (s) & \# Iterations & \# Early Stops & \# Timeouts \\
            \midrule
            \multicolumn{11}{c}{Banknote Authentication} \\
            \midrule
            \multirow{3}{*}{4} & NB  & 0 & 0.238 (0.013) & 4.52 (0.21) & 21 (98.29\%) & 0 & 0.008 (0.001) & 2.00 (0.00) & 1 ($\approx$100.00\%) & 0 \\
            & TAN & 3 & 0.695 (0.023) & 6.84 (0.09) & 25 (98.86\%) & 0 & 0.015 (0.003) & 2.00 (0.00) & 3 (99.97\%) & 0 \\
            & BAN & 5 & 1.167 (0.046) & 7.28 (0.12) & 25 (99.17\%) & 0 & 0.027 (0.004) & 2.00 (0.00) & 0 (\textendash) & 0 \\
            \midrule
            \multicolumn{11}{c}{Pima Indians Diabetes} \\
            \midrule
            \multirow{4}{*}{8} & NB & 0 & 2.231 (0.304) & 5.76 (0.37) & 21 (98.79\%) & 0 & 0.011 (0.004) & 1.92 (0.06) & 0 (\textendash) & 0 \\
            & TAN & 7 & 62.067 (12.248) & 6.80 (0.12) & 25 (98.38\%) & 0 & 0.021 (0.004) & 2.04 (0.04) & 11 ($\approx$100.00\%) & 0 \\
            & BAN$^\dagger$ & 12 & 297.143 (80.996) & 6.80 (0.18) & 24 (98.52\%) & 0 & 0.019 (0.003) & 1.96 (0.04) & 19 ($\approx$100.00\%) & 0 \\
            & BAN & 18 & 853.243 (154.425) & 7.25 (0.09) & 24 (98.55\%) & 1 & 0.051 (0.005) & 2.00 (0.00) & 16 ($\approx$100.00\%) & 0 \\
            \midrule
            \multicolumn{11}{c}{Ionosphere} \\
            \midrule
            \multirow{3}{*}{32$^\ddagger$} & NB & 0  & \textendash & \textendash & \textendash & \textendash & 0.016 (0.001) & 2.92 (0.06) & 4 (99.98\%) & 0 \\
            & TAN & 31 & \textendash & \textendash & \textendash & \textendash & 5.829 (1.229) & 5.67 (0.42) & 19 (99.96\%) & 1 \\
            & BAN$^\dagger$ & 25 & \textendash & \textendash & \textendash & \textendash & 2.684 (2.024) & 3.87 (0.28) & 20 (99.95\%) & 2 \\
            \bottomrule
        \end{tabular}
        \begin{tablenotes}
            \item[$\dagger$] The in-degree of each node is constrained to at most two parent nodes.
            \item[$\ddagger$] Originally containing 34 features, the first two binary-valued features were excluded from the entire pipeline.
        \end{tablenotes}
        \label{tab:bugdet001}
    \end{threeparttable}
\end{table}

Across datasets, feature dependencies substantially increase the runtime of the MILP-based, especially for denser structures such as full BAN, in contrast to the stable performance of the QCQP-based. On Pima Indians Diabetes, moving from NB to BAN raises runtime from seconds to hundreds of seconds, and on the high-dimensional Ionosphere, the MILP-based becomes infeasible given resource limits, underscoring its high sensitivity to structural complexity.\footnote{The BAN with a maximum in-degree of two is constructed by retaining, for each node, the top two incoming edges with the highest absolute weights from the NOTEARS adjacency matrix. All remaining weaker edges are discarded, ensuring that each node has at most two parents.}

An early stop is recorded when the maximum violation $\varphi$ in the final AP iteration lies in $(0, \varepsilon]$, indicating that the counterfactual is not fully robust over the entire uncertainty set. In such cases, we compute $\tilde{\gamma}/\gamma$, where $\tilde{\gamma} \in [0, \gamma)$ is the largest radius around the solution for which all perturbations flip the prediction. While the MILP-based often terminates early, its solutions still achieve 98--99\% radius-level coverage. The QCQP-based also produces early stops in some cases, with ratios effectively reaching 100\%.

\cref{tab:bugdet005} reports results under $\gamma = 0.05$, which exhibit similar patterns to those at $\gamma = 0.01$. Increasing the robustness budget slightly raises runtime, particularly for the MILP-based, but does not reduce coverage, indicating consistent performance across settings. 

\begin{table}[ht]
    \centering
    \begin{threeparttable}
        \tiny
        \caption{Results under robustness budget $\gamma = 0.05$}
        \begin{tabular}{ccccccccccc}
            \toprule
            \multicolumn{3}{c}{} & \multicolumn{4}{c}{Piecewise McCormick Relaxation-Based MILP} & \multicolumn{4}{c}{Original Nonconvex QCQP} \\
            \cmidrule(lr){4-7} \cmidrule(lr){8-11}
            \# Nodes & Structure & \# Edges & Runtime (s) & \# Iterations & \# Early Stops & \# Timeouts & Runtime (s) & \# Iterations & \# Early Stops & \# Timeouts \\
            \midrule
            \multicolumn{11}{c}{Banknote Authentication} \\
            \midrule
            \multirow{3}{*}{4} & NB  & 0 & 0.273 (0.011) & 4.68 (0.17) & 24 (99.65\%) & 0 & 0.010 (0.003) & 2.00 (0.00) & 0 (\textendash) & 0 \\
            & TAN & 3 & 0.950 (0.028) & 7.00 (0.10) & 25 (99.81\%) & 0 & 0.017 (0.004) & 2.00 (0.00) & 3 (99.85\%) & 0 \\
            & BAN & 5 & 1.635 (0.073) & 7.60 (0.15) & 25 (99.86\%) & 0 & 0.028 (0.003) & 2.00 (0.00) & 0 (\textendash) & 0 \\
            \midrule
            \multicolumn{11}{c}{Pima Indians Diabetes} \\
            \midrule
            \multirow{4}{*}{8} & NB & 0 & 2.233 (0.286) & 5.88 (0.30) & 23 (99.76\%) & 0 & 0.011 (0.002) & 2.00 (0.06) & 0 (\textendash) & 0 \\
            & TAN & 7 & 64.358 (12.327) & 6.80 (0.12) & 25 (99.67\%) & 0 & 0.022 (0.003) & 2.12 (0.07) & 14 ($\approx$100.00\%) & 0 \\
            & BAN$^\dagger$ & 12 & 312.133 (82.927) & 6.84 (0.19) & 24 (99.72\%) & 0 & 0.027 (0.005) & 2.08 (0.06) & 15 (99.98\%) & 0 \\
            & BAN & 18 & 872.299 (153.856) & 7.50 (0.10) & 24 (99.79\%) & 1 & 0.073 (0.016) & 2.00 (0.00) & 5 ($\approx$100.00\%) & 0 \\
            \midrule
            \multicolumn{11}{c}{Ionosphere} \\
            \midrule
            \multirow{3}{*}{32$^\ddagger$} & NB & 0  & \textendash & \textendash & \textendash & \textendash & 0.018 (0.002) & 3.00 (0.00) & 17 ($\approx$100.00\%) & 0 \\
            & TAN & 31 & \textendash & \textendash & \textendash & \textendash & 4.028 (0.818) & 11.39 (0.69) & 22 (99.99\%) & 2 \\
            & BAN$^\dagger$ & 25 & \textendash & \textendash & \textendash & \textendash & 5.217 (3.341) & 6.91 (0.48) & 23 (99.99\%) & 2 \\
            \bottomrule
        \end{tabular}
        \begin{tablenotes}
            \item[$\dagger$] The in-degree of each node is constrained to at most two parent nodes.
            \item[$\ddagger$] Originally containing 34 features, the first two binary-valued features were excluded from the entire pipeline.
        \end{tablenotes}
        \label{tab:bugdet005}
    \end{threeparttable}
\end{table}

Overall, the QCQP-based shows instability on the Ionosphere, with some runs resulting in timeouts. This instability is also reflected in the increased iteration counts and longer runtimes observed on this dataset, underscoring the lack of global guarantees and the risk of local optima when directly solving nonconvex formulations.

Another key finding is that the MILP-based approach requires more iterations due to the bound tightening in its piecewise McCormick relaxation. \cref{fig:iterations} illustrates this behavior on the Banknote Authentication with $\gamma = 0.05$ under a full BAN, based on 25 runs per configuration. Increasing the initial partitions reduces iterations by providing a more accurate approximation, but the enlarged model also increases runtime. In contrast, higher contraction factors slow the shrinkage of the search domain, leading to more iterations and longer runtimes.

\begin{figure}[ht]
    \centering
    \includegraphics[scale=0.36]{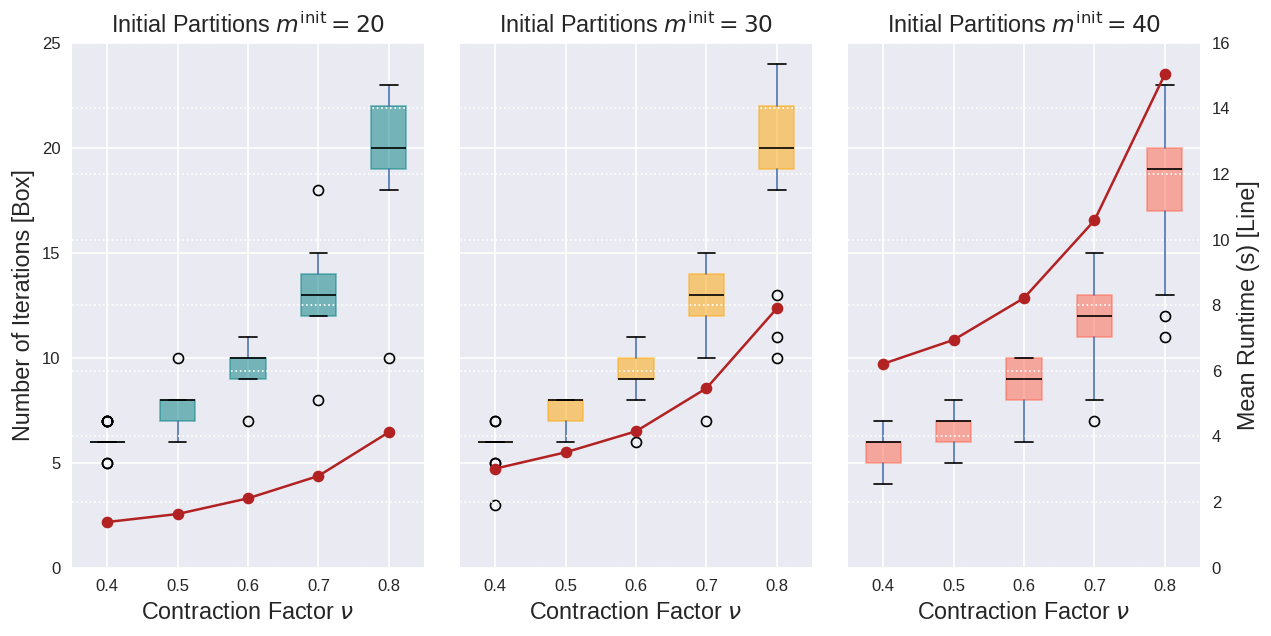}
    \caption{Effect of contraction factor and partitioning on iteration and runtime}
    \label{fig:iterations}
\end{figure}

\cref{fig:runtime_vs_distance} shows the distance of each counterfactual from its factual instance together with the runtime, based on 25 runs on the Pima Indians Diabetes under $\gamma = 0.05$ with a BAN capped at in-degree two. Each dashed line links paired the MILP- and QCQP-based solutions for the same factual instance, indicating that both approaches achieve counterfactuals at similar distances and, in fact, similar locations. The MILP-based runtimes grow exponentially with distance, while the QCQP-based ones remain constant, demonstrating superior scalability.

\begin{figure}[ht]
    \centering
    \includegraphics[scale=0.36]{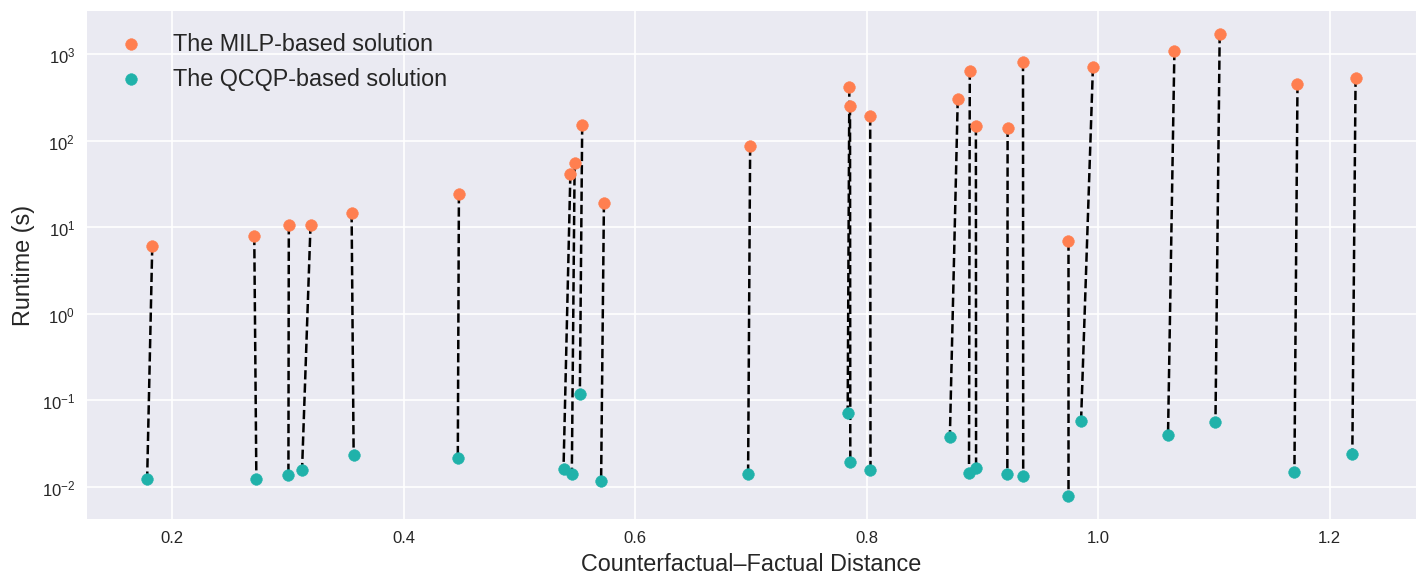}
    \caption{Runtime comparison across counterfactual--factual distances}
    \label{fig:runtime_vs_distance}
\end{figure}

\section{Conclusions and Future Work}
\label{sec:conclusion}

This work introduces a counterfactual search method grounded in CGNCs, embedding conditional feature dependencies directly into the optimization process. The resulting counterfactuals align with model structure and exhibit stronger semantic consistency, while robustness is ensured through an adversarial optimization framework that accounts for implementation uncertainty. The emphasis is on optimization modeling rather than assuming CGNCs fully capture real-world causality.

Our analysis shows that structural complexity fundamentally affects computational efficiency, as the number of quadratic terms grows with feature dependencies. In the MILP-based approach, combinatorial overhead from integer variables inflates problem size and limits scalability, whereas the QCQP-based approach offers more stable runtimes but lacks global guarantees. Since the decision function naturally forms a difference-of-convex (DC) structure, future work could explore DC optimization or hybrid methods to address the underlying theoretical limitations.

As a possible extension, the framework could be applied to mixed discrete--continuous datasets to broaden applicability. In addition, interpretability and plausibility should be assessed from a user perspective, and the causal validity of counterfactuals empirically verified. Addressing these issues will require integrating causal identification methods and human-in-the-loop evaluations to assess practical utility.

Finally, the generative assumptions of CGNCs pose a tension: when class-conditional distributions are well separated, counterfactuals may fall into low-density regions. While these solutions meet classification requirements, their feasibility becomes questionable. Reconciling this trade-off between structural fidelity and semantic plausibility remains a central challenge. This study offers a rigorous yet practical framework and points to future work that combines structural alignment, robustness, and user-centered interpretability.

\appendix
\section{Deviation Representation}
\label{appx:lin}

The coefficients $a_{ij}^c$, defined in \cref{eq:coeff}, correspond to the entries of a matrix $\bm{A}_c = \bm{I} - \bm{W}_c$, where $\bm{W}_c$ encodes the directed linear dependencies among variables under class $c$. Specifically, the matrix $\bm{W}_c$ is defined such that its $(i, j)$-th entry is equal to $w_{ij}^c$ if $X_j$ is a parent of $X_i$, and zero otherwise. When the variables are ordered according to a topological ordering of the underlying DAG, $\bm{W}_c$ becomes strictly lower triangular.

Notably, the matrix $\bm{A}_c$ coincides with the one that appears in the construction of the class-conditional covariance matrix, given by $\bm{\Sigma}_c = \bm{A}_c^{-1} \operatorname{diag}(\sigma_{1\mid c}^2, \sigma_{2\mid c}^2, \dots, \sigma_{n\mid c}^2) {\bm{A}_c^{-\top}}$. More importantly, the $i$-th row vector ${\bm{a}_i^c}^\top$ of $\bm{A}_c$ provides an affine representation of the conditional deviation:
\begin{equation}
    x_i - \hat{\mu}_{i\mid c} = {\bm{a}_i^c}^{\top} \bm{x} - b_i^c \eqcomma
\end{equation}
which serves as the foundation for the modeling of conditional derivations that follow.

\section{Quadratic Structure} 

By examining the form of the decision function $H$ and referring to \cref{eq:prop_logclf}, we observe that, for each class $c$, the only term in the unnormalized log-likelihood $\log h_c(\bm{x})$ that depends on $\bm{x}$ is the sum of squared conditional deviations. All other contributions do not vary with $\bm{x}$ and can therefore be absorbed into a single constant $C_c$. Accordingly, $\log h_c(\bm{x})$ simplifies to
\begin{align}
    \label{eq:exp_log}
    \begin{split}
        \log h_c(\bm{x}) & = \log \prob(Y=c) + \sum_{i=1}^{n} \left( - \frac{1}{2} \log (2 \pi \sigma_{i \mid c}^2) - \frac{{(x_i - \hat{\mu}_{i \mid c})}^2}{2 \sigma_{i\mid c}^2} \right) \\
        & = - \sum_{i=1}^n \frac{{\left({\bm{a}_i^c}^\top \bm{x} - b_i^c \right)}^2}{2\sigma_{i \mid c}^2} + C_c \eqcomma
    \end{split}
\end{align}
where the deviation term is rewritten following the affine form introduced in \cref{appx:lin}.

\section{Quadratic Expansion under Perturbations}
\label{appx:quad_expand}

In the implementation of robust counterfactual search, evaluating the decision function $H$ requires computing the unnormalized joint log-likelihood under perturbation for each class $c$:
\begin{equation}
    \log h_c(\bm{x} + \bm{\delta}) = - \sum_{i=1}^n \frac{{\left({\bm{a}_i^c}^\top (\bm{x} + \bm{\delta}) - b_i^c \right)}^2}{2\sigma_{i\mid c}^2} + C_c \eqperiod
\end{equation}
This implies that the perturbation also alters its conditional mean via changes in its parents. As a result, the perturbed conditional deviation becomes a function of both $\bm{x}$ and $\bm{\delta}$. We then define the squared conditional deviation term:
\begin{equation}
    \begin{split}
        D_i^c(\bm{x}; \bm{\delta}) & \coloneq {\left( {\bm{a}_i^c}^\top (\bm{x} + \bm{\delta}) - b_i^c \right)}^2 \\
        &= {\left( {\bm{a}_i^c}^\top \bm{x} + {\bm{a}_i^c}^\top \bm{\delta} - b_i^c \right)}^2 \eqperiod
    \end{split}
\end{equation}

When solving the MP, we treat $\hat{\bm{\delta}}$ as fixed and optimize over $\bm{x}$. By denoting the constant term $\xi_i^c(\hat{\bm{\delta}}) = {\bm{a}_i^c}^\top \hat{\bm{\delta}} - b_i^c$, in this case, $D_i^c(\bm{x}; \hat{\bm{\delta}})$ can be expanded as a quadratic polynomial in $\bm{x}$:
\begin{align}
    \begin{split}
        D_i^c(\bm{x}; \hat{\bm{\delta}}) &= {\left( {\bm{a}_i^c}^\top \bm{x} + \xi_i^c(\hat{\bm{\delta}}) \right)}^2 \\
        &= \bm{x}^\top \bm{a}_i^c {\bm{a}_i^c}^\top \bm{x} + 2 \xi_i^c(\hat{\bm{\delta}}) \cdot {\bm{a}_i^c}^\top \bm{x} + {\xi_i^c(\hat{\bm{\delta}})}^2 \eqperiod
    \end{split}
\end{align}
Conversely, in the AP, we fix $\hat{\bm{x}}$ and optimize over $\bm{\delta}$. The expansion becomes a quadratic polynomial in $\bm{\delta}$:
\begin{align}
    D_i^c(\bm{\delta}; \hat{\bm{x}}) = \bm{\delta}^\top \bm{a}_i^c {\bm{a}_i^c}^\top \bm{\delta} + 2 \xi_i^c(\hat{\bm{x}}) \cdot {\bm{a}_i^c}^\top \bm{\delta} + {\xi_i^c(\hat{\bm{x}})}^2 \eqperiod
\end{align}

In both formulations, the occurrence of second-order and first-order terms involving the variables is determined by the nonzero structure of $\bm{a}_i^c$. To reduce the complexity of the resulting optimization models, interactions that do not involve node $X_i$ and its parents can be safely ignored, with the relevant indices captured by the set $\paset_i^+$ in \cref{eq:indexset}.

\section{Nonconvexity of the Decision Surface}
\label{appx:quad_nonconvex}

To characterize the curvature of $\log h_c(\bm{x})$, we compute its Hessian with respect to $\bm{x}$:
\begin{align}
    \begin{split}
        \nabla^2 \log h_c(\bm{x}) &= \nabla^2 \left( - \sum_{i=1}^n \frac{{\left({\bm{a}_i^c}^\top \bm{x} - b_i^c \right)}^2}{2\sigma_{i\mid c}^2}  + C_c \right) \\
        &= - \sum_{i=1}^{n} \frac{\bm{a}_i^c {\bm{a}_i^c}^\top}{\sigma_{i\mid c}^2} \preceq 0 \eqcomma
    \end{split}
\end{align}
which is manifestly negative semi-definite, demonstrating that $\log h_c$ is concave in $\bm{x}$.

By contrast, the decision function $H$ combines two concave components with opposite signs. Its Hessian with respect to $\bm{x}$ is
\begin{align}
    \begin{split}
        \nabla^2 H(\bm{x}) & = \nabla^2 \log h_1(\bm{x}) - \nabla^2 \log h_0(\bm{x}) \\
        & = - \sum_{i=1}^{n} \frac{\bm{a}_i^1 {\bm{a}_i^1}^\top}{\sigma_{i\mid 1}^2} + \sum_{i=1}^{n} \frac{\bm{a}_i^0 {\bm{a}_i^0}^\top}{\sigma_{i\mid 0}^2} \eqcomma
    \end{split}
\end{align}
which is the difference of two positive semi-definite matrices. Because such a difference is typically indefinite, $H(\bm{x})$ admits no global convexity or concavity. Accordingly, this makes the counterfactual search a nonconvex optimization problem. An example of such a nonconvex surface, derived from the log-relative likelihood, is illustrated in \cref{fig:nonconvex}.

%\pgfmathdeclarefunction{zdiff}{2}{%
%    \pgfmathsetmacro{\x}{#1}
%    \pgfmathsetmacro{\y}{#2}
%    \pgfmathparse{
%        -0.5*ln(1.35/1.04)
%        -0.5*(1/1.35)*(1.6*\x*\x - 0.6*\x*\y + 0.9*\y*\y)
%        +0.5*(1/1.04)*(0.6*(\x-0.5)^2 -0.4*(\x-0.5)*(\y-0.5) + 1.8*(\y-0.5)^2)
%    }
%}

%\begin{figure}[ht]
%    \centering
%    \begin{tikzpicture}
%        \begin{axis}[
%            view={30}{30},
%            domain=-8:8, domain y=-8:8,
%            samples=6, samples y=6,
%            xlabel=$x_1$, ylabel=$x_2$, zlabel=$H(\bm{x})$,
%            colormap/viridis
%        ] \addplot3[surf] {zdiff(x,y)};
%        \end{axis}
%    \end{tikzpicture}
%    \caption{Surface of the log-relative likelihood $H(\bm{x})$ over two input dimensions}
%    \label{fig:nonconvex}
%\end{figure}

\begin{figure}[ht]
    \centering
    \includegraphics[scale=0.5]{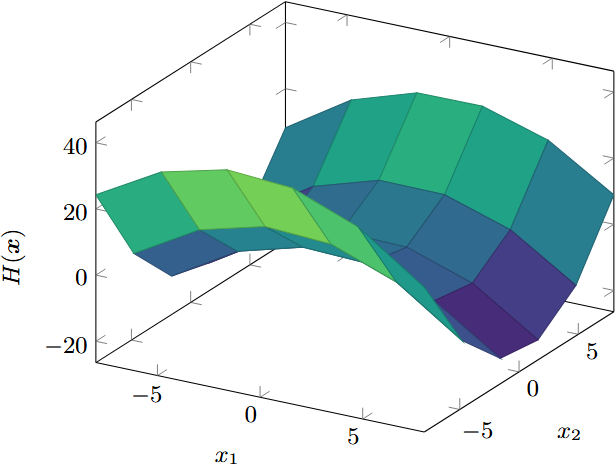}
    \caption{Surface of the log-relative likelihood $H(\bm{x})$ over two input dimensions}
    \label{fig:nonconvex}
\end{figure}

\bibliographystyle{plain}
\bibliography{robustCE}

@misc{karimi2021b,
      author = {Amir-Hossein Karimi and Gilles Barthe and Bernhard Sch{\"o}lkopf and Isabel Valera},
      title = {A Survey of Algorithmic Recourse: Definitions, Formulations, Solutions, and Prospects},
      year = {2021},
      eprint = {2010.04050},
      archivePrefix = {arXiv},
      primaryClass = {cs.LG},
      url = {https://arxiv.org/abs/2010.04050}, 
}

@misc{mishra2023,
      author = {Saumitra Mishra and Sanghamitra Dutta and Jason Long and Daniele Magazzeni},
      title = {A Survey on the Robustness of Feature Importance and Counterfactual Explanations},
      year = {2023},
      eprint = {2111.00358},
      archivePrefix = {arXiv},
      primaryClass = {cs.LG},
      url = {https://arxiv.org/abs/2111.00358}, 
}

@inproceedings{lundberg2017,
               author = {Lundberg, Scott M. and Lee, Su-In},
               title = {A Unified Approach to Interpreting Model Predictions},
               year = {2017},
               publisher = {Curran Associates, Inc.},
               booktitle = {Proceedings of the 31st International Conference on Neural Information Processing Systems},
               pages = {4768--4777}
}

@inproceedings{ustun2019,
               author = {Ustun, Berk and Spangher, Alexander and Liu, Yang},
               title = {Actionable Recourse in Linear Classification},
               year = {2019},
               publisher = {Association for Computing Machinery},
               url = {https://doi.org/10.1145/3287560.3287566},
               booktitle = {Proceedings of the Conference on Fairness, Accountability, and Transparency},
               pages = {10--19}
}

@inproceedings{karimi2021a,
               author = {Karimi, Amir-Hossein and Sch{\"o}lkopf, Bernhard and Valera, Isabel},
               title = {Algorithmic Recourse: From Counterfactual Explanations to Interventions},
               year = {2021},
               publisher = {Association for Computing Machinery},
               url = {https://doi.org/10.1145/3442188.3445899},
               booktitle = {Proceedings of the 2021 ACM Conference on Fairness, Accountability, and Transparency},
               pages = {353--362}
}

@inproceedings{friedman1998,
               author = {Friedman, Nir and Goldszmidt, Mois{\'e}s and Lee, Thomas J.},
               title = {Bayesian Network Classification with Continuous Attributes: Getting the Best of Both Discretization and Parametric Fitting},
               year = {1998},
               publisher = {Morgan Kaufmann Publishers, Inc.},
               booktitle = {Proceedings of the 15th International Conference on Machine Learning},
               pages = {179--187}
}

@article{friedman1997,
         author = {Friedman, Nir and Geiger, Dan and Goldszmidt, Moises},
         title = {Bayesian Network Classifiers},
         year = {1997},
         publisher = {Kluwer Academic Publishers},
         volume = {29},
         url = {https://doi.org/10.1023/A:1007465528199},
         journal = {Machine Learning},
         pages = {131--163}
}

@article{pearl2009,
         author = {Judea Pearl},
         title = {Causal Inference in Statistics: An Overview},
         volume = {3},
         journal = {Statistics Surveys},
         pages = {96--146},
         year = {2009},
         url = {https://doi.org/10.1214/09-SS057}
}

@inproceedings{cheng1999,
               author = {Cheng, Jie and Greiner, Russell},
               title = {Comparing Bayesian Network Classifiers},
               year = {1999},
               publisher = {Morgan Kaufmann Publishers, Inc.},
               booktitle = {Proceedings of the 15th Conference on Uncertainty in Artificial Intelligence},
               pages = {101--108}
}

@article{mccormick1976,
         author = {McCormick, Garth P.},
         title = {Computability of Global Solutions to Factorable Nonconvex Programs: Part {I} -- Convex Underestimating Problems},
         year = {1976},
         volume = {10},
         number = {1},
         url = {https://doi.org/10.1007/BF01580665},
         journal = {Mathematical Programming},
         pages = {147--175}
}

@article{guidotti2022,
         author={Guidotti, Riccardo},
         title={Counterfactual Explanations and How to Find Them: Literature Review and Benchmarking},
         journal={Data Mining and Knowledge Discovery},
         volume={38},
         number={5},
         pages={2770--2824},
         year={2022},
         publisher={Kluwer Academic Publishers},
         url = {https://doi.org/10.1007/s10618-022-00831-6}
}

@inproceedings{slack2021,
               author = {Dylan Slack and Anna Hilgard and Himabindu Lakkaraju and Sameer Singh},
               title = {Counterfactual Explanations can be Manipulated},
               booktitle = {Advances in Neural Information Processing Systems 34 (NeurIPS 2021)},
               pages = {62--75},
               volume = {34},
               year = {2021},
               publisher = {Curran Associates, Inc.}
}

@article{wachter2018,
         author = {Sandra Wachter and Brent Mittelstadt and Chris Russell},
         title = {Counterfactual Explanations without Opening the Black Box: Automated Decisions and the {GDPR}},
         journal = {Harvard Journal of Law \& Technology},
         volume = {31},
         number = {2},
         pages = {841--887},
         year = {2018}
}

@article{mutapcic2009,
         author = {Mutapcic, Almir and Boyd, Stephen},
         title = {Cutting-set Methods for Robust Convex Optimization with Pessimizing Oracles},
         year = {2009},
         volume = {24},
         number = {3},
         url = {https://doi.org/10.1080/10556780802712889},
         journal = {Optimization Methods and Software},
         pages = {381--406}
}

@inproceedings{zheng2018,
               author = {Zheng, Xun and Aragam, Bryon and Ravikumar, Pradeep and Xing, Eric P.},
               title = {{DAG}s with {NO TEARS}: Continuous Optimization for Structure Learning},
               year = {2018},
               publisher = {Curran Associates, Inc.},
               booktitle = {Proceedings of the 32nd International Conference on Neural Information Processing Systems},
               pages = {9492--9503}
}

@inproceedings{kanamori2020,
               title = {{DACE}: Distribution-Aware Counterfactual Explanation by Mixed-Integer Linear Optimization},
               author = {Kanamori, Kentaro and Takagi, Takuya and Kobayashi, Ken and Arimura, Hiroki},
               booktitle = {Proceedings of the 29th International Joint Conference on Artificial Intelligence},
               publisher = {International Joint Conferences on Artificial Intelligence Organization},
               pages = {2855--2862},
               year = {2020},
               url = {https://doi.org/10.24963/ijcai.2020/395}
}

@inproceedings{russell2019,
               author = {Russell, Chris},
               title = {Efficient Search for Diverse Coherent Explanations},
               year = {2019},
               publisher = {Association for Computing Machinery},
               url = {https://doi.org/10.1145/3287560.3287569},
               booktitle = {Proceedings of the Conference on Fairness, Accountability, and Transparency},
               pages = {20--28}
}

@inproceedings{john1995,
               author = {John, George H. and Langley, Pat},
               title = {Estimating Continuous Distributions in Bayesian Classifiers},
               year = {1995},
               publisher = {Morgan Kaufmann Publishers, Inc.},
               booktitle = {Proceedings of the 11th Conference on Uncertainty in Artificial Intelligence},
               pages = {338--345}
}

@article{saeed2023,
         author = {Waddah Saeed and Christian Omlin},         
         title = {Explainable {AI} ({XAI}): A systematic meta-survey of current challenges and future opportunities},
         journal = {Knowledge-Based Systems},
         volume = {263},
         pages = {110273},
         year = {2023},
         url = {https://doi.org/10.1016/j.knosys.2023.110273}
}

@article{barredoarrieta2020,
         title = {Explainable Artificial Intelligence ({XAI}): Concepts, Taxonomies, Opportunities and Challenges Toward Responsible {AI}},
         journal = {Information Fusion},
         volume = {58},
         pages = {82--115},
         year = {2020},
         url = {https://doi.org/10.1016/j.inffus.2019.12.012},
         author = {Alejandro {Barredo Arrieta} and Natalia D{\'i}az-Rodr{\'i}guez and Javier {Del Ser} and Adrien Bennetot and Siham Tabik and Alberto Barbado and Salvador Garcia and Sergio Gil-Lopez and Daniel Molina and Richard Benjamins and Raja Chatila and Francisco Herrera}
}

@article{vandervelden2022,
         author = {van der Velden, Bas H. M. and Kuijf, Hugo J. and Gilhuijs, Kenneth G. A. and Viergever, Max A.},
         title = {Explainable Artificial Intelligence ({XAI}) in Deep Learning-Based Medical Image Analysis},
         journal = {Medical Image Analysis},
         volume = {79},
         pages = {102470},
         year = {2022},
         url = {https://doi.org/10.1016/j.media.2022.102470}  
}

@inproceedings{pawelczyk2022,
               title = {Exploring Counterfactual Explanations Through the Lens of Adversarial Examples: A Theoretical and Empirical Analysis},
               author = {Pawelczyk, Martin and Agarwal, Chirag and Joshi, Shalmali and Upadhyay, Sohini and Lakkaraju, Himabindu},
               booktitle = {Proceedings of the 25th International Conference on Artificial Intelligence and Statistics},
               pages = {4574--4594},
               year = {2022},
               volume = {151},
               publisher = {PMLR},
               url = {https://proceedings.mlr.press/v151/pawelczyk22a.html}
}

@article{cerneviciene2024,
         author = {Jurgita {\v{C}}ernevi{\v{c}}ien{\.e} and Audrius Kaba{\v{s}}inskas},
         title = {Explainable Artificial Intelligence ({XAI}) in Finance: A Systematic Literature Review},
         journal = {Artificial Intelligence Review},
         volume = {57},
         number = {8},
         pages = {1--45},
         year = {2024},
         url = {https://doi.org/10.1007/s10462-024-10854-8}
}

@article{maragno2024,
         author = {Maragno, Donato and Kurtz, Jannis and R{\"o}ber, Tabea E. and Goedhart, Rob and Birbil, {\c{S}}. {\.I}lker and den Hertog, Dick},
         title = {Finding Regions of Counterfactual Explanations via Robust Optimization},
         journal = {INFORMS Journal on Computing},
         volume = {36},
         number = {5},
         pages = {1316--1334},
         year = {2024},
         url = {https://doi.org/10.1287/ijoc.2023.0153}
}

@article{karuppiah2006,
         title = {Global Optimization for the Synthesis of Integrated Water Systems in Chemical Processes},
         journal = {Computers \& Chemical Engineering},
         volume = {30},
         number = {4},
         pages = {650--673},
         year = {2006},
         url = {https://doi.org/10.1016/j.compchemeng.2005.11.005},
         author = {Ramkumar Karuppiah and Ignacio E. Grossmann}
}

@manual{gurobi2024,
        organization = {Gurobi Optimization, LLC},
        title = {Gurobi Optimizer Reference Manual},
        year = {2024},
        url  = {https://www.gurobi.com},
}

@article{ghorbani2019,
         author = {Amirata Ghorbani and Abubakar Abid and James Zou},
         title = {Interpretation of Neural Networks is Fragile},
         journal = {Proceedings of the AAAI Conference on Artificial Intelligence},
         volume = {33},
         number = {1},
         pages = {3681--3688},
         year = {2019},
         url = {https://doi.org/10.1609/aaai.v33i01.33013681}
}

@misc{laugel2019,
      title = {Issues with Post-Hoc Counterfactual Explanations: A Discussion}, 
      author = {Thibault Laugel and Marie-Jeanne Lesot and Christophe Marsala and Marcin Detyniecki},
      year = {2019},
      eprint = {1906.04774},
      archivePrefix = {arXiv},
      primaryClass = {cs.LG},
      url = {https://arxiv.org/abs/1906.04774}, 
}

@book{weaver2018,
      author = {Weaver, Nik},
      title = {Lipschitz Algebras},
      edition = {2nd},
      publisher = {World Scientific},
      year = {2018},
      url = {https://doi.org/10.1142/9911},
      pages = {1--458}
}

@inproceedings{honorio2011,
               author = {Honorio, Jean},
               title = {Lipschitz Parametrization of Probabilistic Graphical Models},
               year = {2011},
               publisher = {AUAI Press},
               booktitle = {Proceedings of the Twenty-Seventh Conference on Uncertainty in Artificial Intelligence},
               pages = {347--354}
}

@article{bergamini2005,
         title = {Logic-Based Outer Approximation for Globally Optimal Synthesis of Process Networks},
         journal = {Computers \& Chemical Engineering},
         volume = {29},
         number = {9},
         pages = {1914--1933},
         year = {2005},
         url = {https://doi.org/10.1016/j.compchemeng.2005.04.003},
         author = {Maria Lorena Bergamini and Pio Aguirre and Ignacio Grossmann}
}

@book{murphy2012,
      title={Machine Learning: A Probabilistic Perspective},
      author={Murphy, Kevin P},
      year={2012},
      publisher={The MIT press}
}

@article{kulis2013,
        author = {Brian Kulis},
        title = {Metric Learning: A Survey},
        journal = {Foundations and Trends in Machine Learning},
        volume = {5},
        number = {4},
        pages = {287--364},
        year = {2013},
        url = {https://doi.org/10.1561/2200000019}
}

@misc{justin2025,
      title = {Responsible Machine Learning via Mixed-Integer Optimization}, 
      author = {Nathan Justin and Qingshi Sun and Andr{\'e}s G{\'o}mez and Phebe Vayanos},
      year = {2025},
      eprint = {2505.05857},
      archivePrefix = {arXiv},
      primaryClass = {cs.LG},
      url = {https://arxiv.org/abs/2505.05857}, 
}

@inproceedings{karimi2020,
               author = {Karimi, Amir-Hossein and Barthe, Gilles and Balle, Borja and Valera, Isabel},
               title = {Model-Agnostic Counterfactual Explanations for Consequential Decisions},
               booktitle = {Proceedings of the 23rd International Conference on Artificial Intelligence and Statistics (AISTATS 2020)},
               pages = {895--905},
               year = {2020},
               volume = {108},
               publisher = {PMLR},
               url = {https://proceedings.mlr.press/v108/karimi20a.html}
}

@article{kanamori2021,
         author= {Kanamori, Kentaro and Takagi, Takuya and Kobayashi, Ken and Ike, Yuichi and Uemura, Kento and Arimura, Hiroki},
         title = {Ordered Counterfactual Explanation by Mixed-Integer Linear Optimization},
         journal = {Proceedings of the AAAI Conference on Artificial Intelligence},
         volume = {35},
         number = {13},
         pages = {11564--11574},
         year = {2021},
         url = {https://doi.org/10.1609/aaai.v35i13.17376}
}

@inproceedings{dominguezolmedo2022,
               title = {On the Adversarial Robustness of Causal Algorithmic Recourse},
               author = {Dominguez-Olmedo, Ricardo and Karimi, Amir-Hossein and Sch{\"o}lkopf, Bernhard},
               booktitle = {Proceedings of the 39th International Conference on Machine Learning},
               pages = {5324--5342},
               year = {2022},
               volume = {162},
               publisher = {PMLR},
               url = {https://proceedings.mlr.press/v162/dominguez-olmedo22a.html}
}

@article{virgolin2023,
         title = {On the Robustness of Sparse Counterfactual Explanations to Adverse Perturbations},
         journal = {Artificial Intelligence},
         volume = {316},
         pages = {103840},
         year = {2023},
         url = {https://doi.org/10.1016/j.artint.2022.103840},
         author = {Marco Virgolin and Saverio Fracaros}
}

@article{deng2021,
         title = {Optimal Operation of Integrated Heat and Electricity Systems: A Tightening {M}c{C}ormick Approach},
         author = {Lirong Deng and Hongbin Sun and Baoju Li and Yong Sun and Tianshu Yang and Xuan Zhang},
         journal = {Engineering},
         volume = {7},
         number = {8},
         pages = {1076--1086},
         year = {2021},
         url = {https://doi.org/10.1016/j.eng.2021.06.006}
}

@article{kessy2018,
         title = {Optimal Whitening and Decorrelation},
         volume = {72},
         url = {http://doi.org/10.1080/00031305.2016.1277159},
         number = {4},
         journal = {The American Statistician},
         publisher = {Informa UK Limited},
         author = {Kessy, Agnan and Lewin, Alex and Strimmer, Korbinian},
         year = {2018},
         pages = {309--314}
}

@article{hasan2010,
         title = {Piecewise Linear Relaxation of Bilinear Programs Using Bivariate Partitioning},
         author = {Hasan, M. M. Faruque and Karimi, Iftekhar A},
         journal = {AIChE journal},
         volume = {56},
         number = {7},
         pages = {1880--1893},
         year = {2010},
         url = {https://doi.org/10.1002/aic.12109}
}

@misc{mahajan2020,
      title = {Preserving Causal Constraints in Counterfactual Explanations for Machine Learning Classifiers}, 
      author = {Divyat Mahajan and Chenhao Tan and Amit Sharma},
      year = {2020},
      eprint = {1912.03277},
      archivePrefix = {arXiv},
      primaryClass = {cs.LG},
      url = {https://arxiv.org/abs/1912.03277}, 
}

@book{koller2009,
      author = {Koller, Daphne and Friedman, Nir},
      title = {Probabilistic Graphical Models: Principles and Techniques},
      year = {2009},
      publisher = {The MIT Press}
}

@inproceedings{albini2020,
               title = {Relation-Based Counterfactual Explanations for Bayesian Network Classifiers},
               author = {Albini, Emanuele and Rago, Antonio and Baroni, Pietro and Toni, Francesca},
               booktitle = {Proceedings of the Twenty-Ninth International Joint Conference on Artificial Intelligence},
               publisher = {International Joint Conferences on Artificial Intelligence Organization},
               pages = {451--457},
               year = {2020},
               url = {https://doi.org/10.24963/ijcai.2020/63},
}

@book{bental2009,
      author = {Aharon Ben-Tal and Laurent El Ghaoui and Arkadi Nemirovski},
      year = {2009},
      title = {Robust Optimization},
      publisher = {Princeton University Press},
      url = {https://doi.org/10.1515/9781400831050}
}

@article{perez2006,
         title = {Supervised Classification with Conditional Gaussian Networks: Increasing the Structure Complexity from Naive Bayes},
         journal = {International Journal of Approximate Reasoning},
         volume = {43},
         number = {1},
         pages = {1--25},
         year = {2006},
         url = {https://doi.org/10.1016/j.ijar.2006.01.002},
         author = {Aritz P{\'e}rez and Pedro Larra{\~n}aga and I{\~n}aki Inza}
}

@article{deeks2019,
         url = {https://www.jstor.org/stable/26810851},
         author = {Ashley Deeks},
         journal = {Columbia Law Review},
         number = {7},
         pages = {1829--1850},
         publisher = {Columbia Law Review Association, Inc.},
         title = {The Judicial Demand for Explainable Artificial Intelligence},
         volume = {119},
         year = {2019}
}

@article{demaesschalck2000,
         title = {The {M}ahalanobis Distance},
         journal = {Chemometrics and Intelligent Laboratory Systems},
         volume = {50},
         number = {1},
         pages = {1--18},
         year = {2000},
         url = {https://doi.org/10.1016/S0169-7439(99)00047-7},
         author = {R. {De Maesschalck} and D. Jouan-Rimbaud and D. L. Massart}
}

@misc{kelly2025,
      author = {Markelle Kelly and Rachel Longjohn and Kolby Nottingham},
      title = {The {UCI} Machine Learning Repository},
      year = {2025},
      url = {https://archive.ics.uci.edu}
}

@inproceedings{ribeiro2016,
               author = {Ribeiro, Marco Tulio and Singh, Sameer and Guestrin, Carlos},
               title = {`{W}hy Should {I} Trust You?': Explaining the Predictions of Any Classifier},
               year = {2016},
               publisher = {Association for Computing Machinery},
               url = {https://doi.org/10.1145/2939672.2939778},
               booktitle = {Proceedings of the 22nd ACM SIGKDD International Conference on Knowledge Discovery and Data Mining},
               pages = {1135--1144}
}

\end{document}